\documentclass[mnsc,non-blindrev]{informs3} 

\OneAndAHalfSpacedXI 

\usepackage{endnotes}
\usepackage[ruled]{algorithm}
\usepackage{algorithmicx}    
\usepackage{algpseudocode}
\usepackage{graphicx}
\usepackage{epstopdf}
\usepackage{subfigure,caption,enumitem,thm-restate}
\usepackage{enumitem}
\usepackage{pifont}
\newcommand{\cmark}{\ding{52}}%
\newcommand{\xmark}{\ding{56}}%
\setlist{leftmargin=*}
\let\footnote=\endnote

\usepackage{natbib}
 \bibpunct[, ]{(}{)}{,}{a}{}{,}%
 \def\bibfont{\small}%
\usepackage{bm}
\usepackage{comment}

\newcommand{\rad}{\mathsf{rad}\text{-}}
\newcommand{\AAA}{\mathcal{A}}
\renewcommand{\H}{\mathcal{H}}
\newcommand{\SSS}{\mathcal{S}}
\renewcommand{\Re}{\mathbb{R}}
\newcommand{\dev}{\mathsf{var}}
\newcommand{\reward}{\textsf{R}}
\newcommand{\pse}{\text{pseudo}}
\newcommand{\ie}{\textit{i.e.}}

\newcommand{\eg}{\textit{e.g.}}
\newcommand{\sw}{\texttt{SWUCRL2-CW} algorithm}
\newcommand{\borl}{\texttt{BORL} algorithm}
\newcommand{\E}{\mathbb{E}}
\definecolor{darkolivegreen}{rgb}{0.33, 0.42, 0.18}
\DeclareSymbolFont{extraup}{U}{zavm}{m}{n}
\DeclareMathSymbol{\varheart}{\mathalpha}{extraup}{86}
\DeclareMathSymbol{\vardiamond}{\mathalpha}{extraup}{87}

\newtheorem{property}{Property}

\TheoremsNumberedThrough     
\ECRepeatTheorems

\EquationsNumberedThrough    

\allowdisplaybreaks
\begin{document}


\RUNAUTHOR{Cheung, Simchi-Levi, and Zhu}

\RUNTITLE{Drifting Reinforcement Learning}

\TITLE{Non-Stationary Reinforcement Learning: The Blessing of (More) Optimism}

\ARTICLEAUTHORS{
\AUTHOR{Wang Chi Cheung}
\AFF{Department of Industrial Systems Engineering and Management, National University of Singapore \\
	\EMAIL{isecwc@nus.edu.sg}}
\AUTHOR{David Simchi-Levi}
\AFF{Institute for Data, Systems, and Society, Massachusetts Institute of Technology, Cambridge, MA 02139\\ \EMAIL{dslevi@mit.edu}}
\AUTHOR{Ruihao Zhu}
\AFF{Institute for Data, Systems, and Society, Massachusetts Institute of Technology, Cambridge, MA 02139\\
	 \EMAIL{rzhu@mit.edu}} 
} 

\ABSTRACT{%
		We consider un-discounted reinforcement learning (RL) in  Markov decision processes (MDPs) under temporal drifts, \ie, both the reward and state transition distributions are allowed to evolve over time, as long as their respective total variations, quantified by suitable metrics, do not exceed certain \emph{variation budgets}. This setting captures endogeneity, exogeneity, uncertainty, and partial feedback in sequential decision-making scenarios, and finds applications in various online marketplaces, epidemic control, and transportation. We first develop the Sliding Window Upper-Confidence bound for Reinforcement Learning with Confidence Widening (\texttt{SWUCRL2-CW}) algorithm, and establish its \emph{dynamic regret} bound when the variation budgets are known. In addition, we propose the Bandit-over-Reinforcement Learning (\texttt{BORL}) algorithm to adaptively tune the \sw~to achieve the same dynamic regret bound, but  in a \emph{parameter-free} manner, \ie, without knowing the variation budgets. Finally, we conduct numerical experiments to show that our proposed algorithms achieve superior empirical performance compared with existing algorithms.
		
		Notably, the interplay between endogeneity and exogeneity presents a unique challenge, absent in existing (stationary and non-stationary) stochastic online learning settings, when one applies the conventional Optimism in Face of Uncertainty (OFU) principle to design algorithms with provably low dynamic regret for RL in non-stationary MDPs. We overcome this challenge by a novel confidence widening technique that incorporates additional optimism into our learning algorithms to ensure low dynamic regret bounds. To extend our theoretical findings, we demonstrate, in the context of single item inventory control with fixed cost, how one can leverage special structures on the state transition distributions to bypass the difficulty of exploring time-varying environments.
}%



\maketitle

%
\section{Introduction}
\label{sec:intro}
Consider a general sequential decision-making framework, where a decision-maker (DM) interacts with an initially unknown environment iteratively. At each time step, the DM first observes the current state of the environment, and then chooses an available action. After that, she receives an instantaneous random reward, and the environment transitions to the next state. The DM aims to design a policy that maximizes its cumulative rewards, while facing the following challenges:
\begin{itemize}
	\item \textbf{Endogeneity:} At each time step, the reward follows a reward distribution, and the subsequent state follows a state transition distribution. Both distributions depend (solely) on the current state and action, which are influenced by the policy. Hence, the environment can be fully characterized by a discrete time Markov decision process (MDP).
	\item \textbf{Exogeneity:}  The reward and state transition distributions vary (independently of the policy) across time steps, but the total variations are bounded by the respective variation budgets.
	\item \textbf{Uncertainty:} Both the reward and state transition distributions are initially unknown to the DM. 
	\item \textbf{Bandit/Partial Feedback:} The DM can only observe the reward and state transition resulted by the current state and action in each time step.
\end{itemize}
\subsection{Motivating Examples}
It turns out that many applications, such as vehicle remarketing in used-car sales and real-time bidding in advertisement (ad) auctions, can be captured by this framework. 
\begin{example}[\textbf{Vehicle Remarketing in Used-Cars Sales}] An automobile company disposes of continually arriving off-lease vehicles (\ie, leasing vehicles that have reached the end of their fixed term) via daily wholesale vehicle auctions \citep{Manheim,VR}. At the beginning of each auction, the company observes the number of on-hand vehicles (the ``state"), and decides the number of off-lease vehicles to be listed (the ``action"). Then, the car dealers bid for the purchases via a first-price auction. The sales of vehicles generate revenue to the company while unsold vehicles incur holding cost to the company (the ``reward" and ``state transition"). The company aims at maximizing profit by designing a policy that dynamically decides the vehicles to be listed in each auction. However, the dealers' bidding behaviors are affected by many unpredictable (and thus exogenous) factors (\eg, real-time customer demands, vehicles' depreciation, and inter-dealer competitions) in addition to the company's decisions (\ie,  the vehicles listed), and can vary across time.
\end{example}
\begin{example}[\textbf{Real-Time Bidding in Ads Auctions}] Advertisers repeatedly competes for ad display impressions via real-time online auctions \citep{Google11,CaiRZ17,FlajoletJ17,BalseiroG19,GuoHXZ19}. Each advertiser begins with a budget. Upon the arrival of a user, an impression is generated, and the advertisers submit bids (the ``action") for it subject to her remaining budget (the ``state"). The winning advertiser acquires the impression to display her ad to the user, and observes the user click or no-click behavior (the ``reward").  For each slot won, the advertiser has to make the payment (determined by the auction mechanism) using her remaining budget, and the budget is periodically refilled (the ``state transition"). Each advertiser wants to maximize the number of clicks on her advertisement subject to her own (continuously evolving) budget constraint. Nevertheless, the competitiveness of each auction exhibits exogeneity as the participating advertisers and the arriving users are different from time to time. Moreover, the popularity of an ad can change due to endogenous reasons. For instance, displaying the same ad too frequently in a short period of time might reduce its freshness, and results in a tentatively low number of clicks (\ie, we can incorporate both the remaining budget and the number of times that the ad is shown within a given window size into the state of the MDP to model endogenous dynamics).
\end{example} 
Besides, this framework can be used to model sequential decision-making problems in transportation \citep{ZhangW18,QinTY19}, wireless network \citep{ZhouB15,ZhouGB16}, consumer choice modeling \citep{XuY20}, healthcare operations \citep{ShortreedLLSPM10}, epidemic control \citep{NowzariPP16,KissMS17},  and inventory control \citep{HuhR09,B17,ZhangCS18,AgrawalJ19,ChenSD19}.
\subsection{Different Scenarios of Sequential Decision-Making}\label{sec:diff_scenario}
There exists numerous works in sequential decision-making that considered part of the four challenges (Please refer to Table \ref{table:challenge} for a summary and comparison). The traditional stream of research \citep{ABF02,BC12,LS18} on stochastic \emph{multi-armed bandits} (MAB) focused on the interplay between uncertainty and bandit feedback (\ie, challenges 3 and 4), and \citep{ABF02} proposed the classical \emph{Upper Confidence Bound} (UCB) algorithm. Starting from \citep{BurnetasK97,TewariB08,JakschOA10}, a volume of works (see Section \ref{sec:related}) have been devoted to \emph{reinforcement learning} (RL) in MDPs \citep{SuttonB18}, which further involves endogeneity. RL in MDPs incorporate challenges 1,3,4, and stochastic MAB is a special case of MDPs when there is only one state. In the absence of exogeneity, the reward and state transition distributions are invariant across time, and these three challenges can be jointly solved by the \emph{Upper Confidence bound for Reinforcement Learning} (UCRL2) algorithm \citep{JakschOA10}. 

The UCB and UCRL2 algorithms leverage the \emph{optimism in face of uncertainty} (OFU) principle to select actions iteratively based on the entire collections of historical data. However, both algorithms quickly deteriorate when exogeneity emerge since the environment can change over time, and the historical data becomes obsolete. To address the challenge of exogeneity,  \citep{GarivierM11} considered the \emph{piecewise-stationary} MAB environment where the reward distributions remain unaltered over certain time periods and change at unknown time steps. Later on, there is a line of research initiated by \citep{BGZ14} that studied the general \emph{non-stationary} MAB environment \citep{BGZ14,CSLZ19,CSLZ19a}, in which the reward distributions can change arbitrarily over time, but the total changes (quantified by a suitable metric) is upper bounded by a \emph{variation budget} \citep{BGZ14}. The aim is to minimize the \emph{dynamic regret}, the optimality gap compared to the cumulative rewards of the sequence of optimal actions. Both the (relatively restrictive) piecewise-stationary MAB and the general non-stationary MAB settings consider the challenges of exogeneity, uncertainty, and partial feedback (\ie, challenges 2, 3, 4), but endogeneity (challenge 1) are not present. 

In this paper, to address all four above-mentioned challenges, we consider RL in \emph{non-stationary} MDPs where bot the reward and state transition distributions can change over time, but the total changes (quantified by suitable metrics) are upper bounded by the respective variation budgets. We note that in \citep{JakschOA10}, the authors also consider the intermediate RL in \emph{piecewise-stationary} MDPs. Nevertheless, we first demonstrate in Section \ref{sec:challenge}, and then rigorously show in Section \ref{sec:cw_disc} that simply adopting the techniques for non-stationary MAB \citep{BGZ14,CSLZ19,CSLZ19a} or RL in piecewise-stationary MDPs \citep{JakschOA10} to RL in non-stationary MDPs may result in poor dynamic regret bounds.
\begin{table}[t]
	\begin{center}
		\renewcommand{\arraystretch}{2}
		\begin{tabular}{|c|c|c|c|c|} 
			\hline
			&Endogeneity&Exogeneity&Uncertainty&Bandit feedback\\
			\hline
			Stationary MAB &\xmark&\xmark&\cmark&\cmark\\
			\hline			
			RL in stationary MDPs &\cmark&\xmark&\cmark&\cmark\\
			\hline 
			Non-stationary MAB&\xmark&\cmark&\cmark&\cmark\\
			\hline			
			RL in non-stationary MDPs&\cmark&\cmark&\cmark&\cmark\\
			\hline 
		\end{tabular}
		\caption{Summary of different sequential decision-making settings. Among them, RL in non-stationary MDPs is the only setting that addresses all four challenges.}\label{table:challenge}
	\end{center}
\end{table}
\subsection{Summary of Main Contributions}\label{sec:contribution} Assuming that, during the $T$ time steps, the total variations of the reward and state transition distributions are bounded (under suitable metrics) by the variation budgets $B_r~(>0)$ and $B_p~(>0),$ respectively, we design and analyze novel algorithms for RL in non-stationary MDPs. Let $D_{\max},$ $S,$ and $A$ be respectively the maximum diameter (a complexity measure to be defined in Section \ref{sec:model}), number of states, and number of actions in the MDP. Our main contributions are:
\begin{itemize}
	\item We develop the Sliding Window UCRL2 with Confidence Widening (\texttt{SWUCRL2-CW}) algorithm. When the variation budgets are known, we prove it attains a $\tilde{O}\left( D_{\max} (B_r + B_p)^{1/4} S^{2/3}A^{1/2}T^{3/4}\right)$ dynamic regret bound via a budget-aware analysis.
	\item We propose the Bandit-over-Reinforcement Learning (\texttt{BORL}) algorithm that tunes the \sw~adaptively, and retains the same $\tilde{O}\left(D_{\max}(B_r+B_p)^{{1}/{4}}S^{{2}/{3}}A^{{1}/{2}}T^{{3}/{4}}\right)$ dynamic regret bound without knowing the variation budgets.
	\item We identify an unprecedented challenge for RL in non-stationary MDPs with conventional optimistic exploration techniques: existing algorithmic frameworks for non-stationary online learning (including non-stationary bandit and RL in piecewise-stationary MDPs) \citep{JakschOA10,GarivierM11,CSLZ19} typically estimate unknown parameters by averaging historical data in a ``forgetting" fashion, and construct the \emph{tightest} possible confidence regions/intervals accordingly. They then optimistically search for the most favorable model within the confidence regions, and execute the corresponding optimal policy. However, we first demonstrate in Section \ref{sec:challenge}, and then rigorously show in Section \ref{sec:cw_disc} that in the context of RL in non-stationary MDPs, the diameters induced by the MDPs in the confidence regions constructed in this manner can grow wildly, and may result in unfavorable dynamic regret bound. We overcome this with our novel proposal of extra optimism via the confidence widening technique. A summary of the algorithmic frameworks for stationary and non-stationary online learning settings are provided in Table \ref{table:summary}.
	\begin{table}[!ht]
			\renewcommand{\arraystretch}{2}
		\begin{tabular}{|c|c|c|} 
			\hline
			&Stationary&Non-stationary\\
			\hline
			MAB&OFU \citep{ABF02}&OFU + Forgetting \citep{BGZ14,CSLZ19a}\\
			\hline			
			RL&OFU \citep{JakschOA10}&\textbf{Extra optimism + Forgetting (This paper)}\\
			\hline 
		\end{tabular}
	\caption{Summary of algorithmic frameworks of stationary and non-stationary online learning settings.}\label{table:summary}
	\end{table}
	\item As a complement to this finding, suppose for any pair of initial state and target state, there always exists an action such that the probability of transiting from the initial state to the target state by taking this action is lower bounded uniformly over the entire time horizon, the DM can attain low dynamic regret without widening the confidence regions. We demonstrate that in the context of single item inventory control with fixed cost \citep{YuanLS19}, a mild condition on the demand distribution is sufficient for this extra assumption to hold.
\end{itemize} 
\subsection{Paper Organization}
The rest of the paper is organized as follows: in Section \ref{sec:model}, we describe the non-stationary MDP model of interest. In Section \ref{sec:related}, we review related works in non-stationary online learning and reinforcement learning. In Section \ref{sec:algo}, we introduce the \sw, and analyze its performance in terms of dynamic regret. In Section \ref{sec:borl}, we design the \borl~that can attain the same dynamic regret bound as the \sw~without knowing the total variations. In Section \ref{sec:cw_disc}, we discuss the challenges in designing learning algorithms for reinforcement learning under drift, and manifest how the novel confidence widening technique can mitigate this issue. In Section \ref{sec:alternative}, we discuss the alternative approach without widening the confidence regions in inventory control problems. In Section \ref{sec:numerical}, we conduct numerical experiments to show the superior empirical performances of our algorithms. In Section \ref{sec:conclusion}, we conclude our paper.

\section{Problem Formulation}\label{sec:model}
In this section, we introduce the notations to be used throughout paper, and introduce the learning protocol for our problem of RL in non-stationary MDPs.
\subsection{Notation}
Throughout the paper, all vectors are column vectors, unless specified otherwise. We define $[n]$ to be the set $\{1,2,\ldots,n\}$ for any positive integer $n.$ We denote $\mathbf{1}[\cdot]$ as the indicator function. For $p\in [1, \infty]$, we use $\|x\|_p$ to denote the $p$-norm of a vector $ x\in\Re^d.$ We denote $x\vee y$ and $x\wedge y$ as the maximum and minimum between $x,y\in\Re,$ respectively. We adopt the asymptotic notations $O(\cdot),\Omega(\cdot),$ and $\Theta(\cdot)$ \citep{CLRS09}. When logarithmic factors are omitted, we use $\tilde{O}(\cdot),\tilde{\Omega}(\cdot),$ $\tilde{\Theta}(\cdot),$ respectively. With some abuse, these notations are used when we try to avoid the clutter of writing out constants explicitly.
\subsection{Learning Protocol}
\textbf{Model Primitives:} An instance of non-stationary MDP is specified by the tuple $(\SSS, \AAA, T, r, p)$. The set $\SSS$ is a finite set of states. 
The collection $\AAA = \{\AAA_s\}_{s\in \SSS}$ contains a finite action set $\AAA_s$ for each state $s\in \SSS$.  We say that $(s, a)$ is a state-action pair if $s\in \SSS, a\in \AAA_s$. We denote $S = |\SSS |$, $A =(\sum_{s\in \SSS} |\AAA_s|)/S.$ We denote $T$ as the total number of time steps, and denote $r = \{r_t\}^T_{t=1}$ as the sequence of mean rewards. For each $t$, we have $r_t = \{r_t(s, a)\}_{s\in \SSS, a\in \AAA_s}$, and $r_t(s, a)\in [0, 1]$ for each state-action pair $(s, a)$.
In addition, we denote $p = \{p_t\}^T_{t=1}$ as the sequence of state transition distributions. For each $t$, we have $p_t = \{p_t(\cdot | s, a)\}_{s\in \SSS, a\in \AAA_s}$, where $p_t(\cdot | s, a)$ is a probability distribution over $\SSS$ for each state-action pair $(s, a)$. 

\noindent\textbf{Exogeneity:} The quantities $r_t$'s and $p_t$'s vary across different $t$'s in general. Following \citep{BGZ14}, we quantify the variations on $r_t$'s and $p_t$'s in terms of their respective \emph{variation budgets} $B_r, B_p~(>0)$:
\begin{align}
B_r &= \sum^{T - 1}_{t=1} B_{r, t}\text{, where }B_{r, t} = \max_{s\in \SSS, a\in \AAA_s}\left| r_{t + 1}(s, a) - r_t(s, a)  \right|,\nonumber\\
B_p &= \sum^{T - 1}_{t=1} B_{p, t}\text{, where }B_{p, t} = \max_{s\in \SSS, a\in \AAA_s }\left\| p_{t + 1}(\cdot | s, a) - p_t ( \cdot  |s, a)  \right\|_1.\label{eq:variation_budgets}
\end{align}
We emphasize although $B_r$ and $B_p$ might be used as inputs by the DM, individual $B_{r,t}$'s and $B_{p,t}$'s are unknown to the DM throughout the current paper. We also refer to Remark \ref{remark:norm} for the choice of infinity-norm and 1-norm in eqn. \eqref{eq:variation_budgets}.

\noindent\textbf{Endogeneity:} The DM faces a non-stationary MDP instance $(\SSS,\AAA, T, r, p)$. She knows $\SSS, \AAA, T$, but not $r, p$. The DM starts at an arbitrary state $s_1\in \SSS$. At time $t$, three events happen. First, the DM observes its current state $s_t$. Second, she takes an action $a_t\in \AAA_{s_t}$. Third, given $s_t, a_t$, she stochastically transits to another state $s_{t+1}$ which is distributed as $p_t(\cdot | s_t, a_t)$, and receives a stochastic reward $R_t(s_t,a_t)$, which is 1-sub-Gaussian with mean $r_t(s_t, a_t)$. In the second event, the choice of $a_t$ is based on a \emph{non-anticipatory} policy $\Pi$. That is, the choice only depends on the current state $s_t$ and the previous observations $\H_{t-1} := \{s_q, a_q, R_q(s_q, a_q)\}^{t-1}_{q =1}$.

\noindent\textbf{Dynamic Regret:} The DM aims to maximize the cumulative expected reward $\E [\sum^T_{t=1} r_t(s_t, a_t) ] $, despite the model uncertainty on $r, p$ and the dynamics of the learning environment. To measure the convergence to optimality, we consider an equivalent objective of minimizing the \emph{dynamic regret} \citep{BGZ14,JakschOA10}
\begin{equation}\label{eq:dyn_reg}
\text{Dyn-Reg}_T(\Pi)= \sum^T_{t = 1}\left\{\rho^*_t - \mathbb{E}[r_t(s_t, a_t) ] \right\}.
\end{equation}
In the oracle $\sum^T_{t = 1}\rho^*_t $, the summand $\rho^*_t$ is the optimal long-term average reward of the stationary MDP with state transition distribution $p_t$ and mean reward $r_t.$  The optimum $\rho^*_t$ can be computed by solving linear program \eqref{eq:primal} provided in Section \ref{sec:mdp_lp}. We note that the same oracle is used for RL in piecewise-stationary MDPs \citep{JakschOA10}.
\begin{remark}[\textbf{Comparisons with Non-Stationary MAB}]
	When $S=1$, eqn. (\ref{eq:dyn_reg}) reduces to the definition \citep{BGZ14} of dynamic regret for non-stationary $K$-armed bandit. Different from the bandit case, however, the oracle $\sum^T_{t = 1}\rho^*_t $ does not equal to the expected optimum for the non-stationary MDP problem in general. Nevertheless, we justify this choice in Proposition \ref{prop:benchmark}.
\end{remark}
\begin{remark}[\textbf{Definition of Variation Budgets}]\label{remark:norm}
	For brevity of exposition, we choose to define the variation budgets (see eqn. \eqref{eq:variation_budgets} ) for reward and state transition distributions with the infinity norm and 1 norm, respectively. One can also define them with respect to other commonly used metrics, such as the 2 norm \citep{CSLZ19}, and the this would only affect the dependence on $S$ and $A$ for the established dynamic regret bounds in the subsequent sections.
\end{remark}
Next, we review relevant concepts on MDPs, in order to stipulate an assumption that ensures learnability and justifies our oracle. 
\begin{definition}[\textbf{Communicating MDPs and Diameter} \citep{JakschOA10}]\label{def:diameter}
	Consider a set of states $\SSS$, a collection $\AAA= \{\AAA_s\}_{s\in \SSS}$ of action sets, and a state transition distribution $\bar{p} = \{\bar{p}(\cdot | s, a)\}_{s\in \SSS, a\in \AAA_s}$. For any $s, s'\in \SSS$ and stationary policy $\pi$, the hitting time from $s$ to $s'$ under $\pi$ is the random variable $\Lambda(s' | \pi, s) := \min\left\{ t : s_{t+1} = s' , s_1 = s, s_{\tau + 1} \sim \bar{p}(\cdot | s_\tau, \pi(s_\tau)) \text{ $\forall\tau$}\right\},$ which can be infinite. We say that $(\SSS, \AAA, \bar{p})$ is a \emph{communicating MDP} iff  $D := \max_{s, s'\in \SSS}\min_{\text{stationary }\pi} \mathbb{E}\left[\Lambda(s' | \pi, s) \right]$ is finite. The quantity $D$ is the \emph{diameter} associated with $(\SSS, \AAA, \bar{p})$.
\end{definition}
\begin{remark}[\textbf{Diameter and RL in MDPs}]
	As shown in \citep{JakschOA10}, ``diameter" plays a fundamental role in characterizing the complexity of RL in MDPs. Intuitively, in order to make informative decisions, the DM has to have accurate estimates of the quantities $r_t(s,a)$'s and $p_t(\cdot|s,a).$'s. In other words, she must visit every state $s\in\SSS$ and choose each of its available actions $a\in\AAA_s$ frequently enough to collect relevant samples. Consequently, the harder to transition from a state $s$ to another state $s',$ the more the DM would suffer during the learning process, and the diameter of a MDP captures the ``hardness" of transitioning between states in this MDP.
\end{remark}
With the above remark, we make the following assumption.
\begin{assumption}[Bounded Diameters]\label{ass:communicating}
	For each $ t \in[T]$, the tuple $(\SSS, \AAA, p_t) $ constitutes a communicating MDP with diameter at most $D_t$. We denote the \emph{maximum diameter} as $D_\text{max} = \max_{t\in \{1, \ldots, T\}} D_t.$
\end{assumption}

The following proposition justifies our choice of oracle $\sum^T_{t=1}\rho^*_t$.
\begin{proposition}\label{prop:benchmark}
	Consider an instance $(\SSS, \AAA, T, p, r)$ that satisfies Assumption \ref{ass:communicating} with maximum diameter $D_\text{max}$, and has variation budgets $B_r, B_p$ for the rewards and transition distributions respectively. In addition, suppose that $T \geq B_r + 2 D_\text{max} B_p > 0,$ then it holds that
	$$
	\sum^T_{t=1}\rho^*_t \geq \max_{\Pi}\left\{\mathbb{E}\left[\sum^T_{t=1} r_t(s^\Pi_t, a^\Pi_t) \right]\right\} - 4(D_\text{max} + 1)\sqrt{(B_r + 2 D_\text{max}B_p)T}.
	$$
	The maximum is taken over all non-anticipatory policies $\Pi$'s. We denote $\{(s^\Pi_t, a^\Pi_t)\}^T_{t=1}$ as the trajectory under policy $\Pi$, where $a^\Pi_t\in \AAA_{s^\Pi_t}$ is determined based on $\Pi$ and $\H_{t-1} \cup \{s^\Pi_t\}$, and $s^\Pi_{t+1}\sim p_t (\cdot | s^\Pi_t, a^\Pi_t)$ for each $t$.
\end{proposition}
The Proposition is proved in section \ref{app:pfprofbenchmark} of the appendix. In fact, our dynamic regret bounds (see the forthcoming Theorems \ref{thm:main1} and \ref{thm:borl}) are larger than the error term $4(D_\text{max} + 1)\sqrt{(B_r + 2 D_\text{max}B_p)T}$, thus justifying the choice of $\sum^T_{t=1}\rho^*_t$ as the oracle. It turns out that the oracle $\sum^T_{t=1}\rho^*_t$ is more convenient for analysis than the expected optimum, since the former can be decomposed to summations across different intervals, unlike the latter where the summands are intertwined due to endogenous dynamics, \ie, $s^\Pi_{t+1} \sim p_t(\cdot | s^\Pi_t, a^\Pi_t)$. 

\section{Related Works}\label{sec:related}
\subsection{RL in Stationary MDPs}
RL in stationary (discounted and un-discounted reward) MDPs has been widely studied in 
\citep{BurnetasK97,BartlettT09,JakschOA10,AgrawalJ17,FruitPL18,FruitPLO18,SidfordWWY18,SidfordWWYY18,Wang19,ZhangJ19,FruitPL19,WeiJLSJ19}. For the discounted reward setting, the authors of \citep{SidfordWWY18,Wang19,SidfordWWYY18} proposed (nearly) optimal algorithms in terms of sample complexity. For the un-discounted reward setting, the authors of \citep{JakschOA10} established a minimax lower bound $\Omega(\sqrt{D_{\max}SAT})$ on the regret when both the reward and state transition distributions are time-invariant. They also designed the UCRL2 algorithm and showed that it attains a regret bound $\tilde{O}(D_{\max}S\sqrt{AT}).$ The authors of \citep{FruitPL19} proposed the UCRL2B algorithm, which is an improved version of the UCRL2 algorithm. The regret bound of the UCRL2B algorithm is $\tilde{O}(S\sqrt{D_{\max}AT}+D_{\max}^2S^2A).$ The minimax optimal algorithm is provided in \citep{ZhangJ19} although it is not computationally efficient.
\subsection{RL in Non-Stationary MDPs}
In a parallel work \citep{GajaneOA19v3}, the authors considered a similar setting to ours by applying the ``forgetting principle" from non-stationary bandit settings \citep{GarivierM11,CSLZ19a} to design a learning algorithm. To achieve its dynamic regret bound, the algorithm by \citep{GajaneOA19v3} partitions the entire time horizon $[T]$ into time intervals ${\cal I}=\{I_k\}^K_{k=1},$ and crucially requires the access to $\sum_{t=\min I_k}^{\max I_k-1}B_{r, t}$ and $\sum_{t=\min I_k}^{\max I_k-1}B_{p, t},$ \ie, the variations in both reward and state transition distributions of each interval $I_k\in {\cal I}$ (see Theorem 3 in \citep{GajaneOA19v3}). In contrast, the \sw~and the \borl~require significantly less information on the variations. Specifically, the \sw~does not need any additional knowledge on the variations except for $B_r$ and $B_p,$ \ie, the variation budgets over the entire time horizon as defined in eqn. \eqref{eq:variation_budgets}, to achieve its dynamic regret bound (see Theorem \ref{thm:main1}). This is similar to algorithms for the non-stationary bandit settings, which only require the access to $B_r$ \citep{BGZ14}. More importantly, the \borl~(built upon the \sw) enjoys the same dynamic regret bound even without knowing either of $B_r$ or $B_p$ (see Theorem \ref{thm:borl}). 

There also exists some settings that are closely related to, but different than our setting (in terms of exogeneity and feedback). \citep{JakschOA10,GajaneOA18} proposed solutions for the RL in piecewise-stationary MDPs setting. But as discussed in Section \ref{sec:diff_scenario}, simply applying their techniques to the general RL in non-stationary MDPs may result in undesirable dynamic regret bounds (see Section \ref{sec:cw_disc} for more details). In \citep{YuMS09,NueAAS10,AroraDT12,DickGC14,JinJLSY19,CardosoWX19}, the authors considered RL in MDPs with changing reward distributions but fixed transition distributions. The authors of \citep{Even-DarKM05,YuM09,NeuGS12,Abbasi-YadkoriBKSS13,RosenbergM19a,LiZQL19} considered RL in non-stationary MDPs with full information feedback.
\subsection{Non-Stationary MAB}
For online learning and bandit problems where there is only one state, the works by \citep{ABFS02,GarivierM11,BGZ14,KZ16} proposed several ``forgetting" strategies for different non-stationary MAB settings. More recently, the works by \citep{KA16,LWAL18,CSLZ19,CSLZ19a,CLLW19} designed parameter-free algorithms for non-stationary MAB problems. Another related but different setting is the Markovian bandit \citep{KimL16,Ma18}, in which the state of the chosen action evolve according to an independent time-invariant Markov chain while the states of the remaining actions stay unchanged. In \citep{ZhouCGX20}, the authors also considered the case when the states of all the actions are governed by the same (uncontrollable) Markov chain.

\section{Sliding Window UCRL2 with Confidence Widening Algorithm}\label{sec:algo}
In this section, we first describe a unique challenge in RL in non-stationary MDPs, and then present the \sw, which incorporates our novel confidence widening technique and sliding window estimates \citep{GarivierM11} into UCRL2 \citep{JakschOA10}. 
\subsection{Design Challenge: Failure of Naive Sliding Window UCRL2 Algorithm}\label{sec:challenge}
For stationary MAB problems, the UCB algorithm \citep{ABF02} suggests the DM should iteratively execute the following two steps in each time step: 
\begin{enumerate}
	\item Estimate the mean reward of each action by taking the time average of \emph{all} observed samples.
	\item Pick the action with the highest estimated mean reward plus the confidence radius, where the radius scales inversely proportional with the number of observations \citep{ABF02}. 
\end{enumerate}
The UCB algorithm has been proved to attain optimal regret bounds for various stationary MAB settings \citep{ABF02,KWAS15}. For non-stationary problems, \citep{GarivierM11,KZ16,CSLZ19a} shown that the DM could further leverage the forgetting principle by incorporating the sliding-window estimator \citep{GarivierM11} into the UCB algorithms \citep{ABF02,KWAS15} to achieve optimal dynamic regret bounds for a wide variety of non-stationary MAB settings. The sliding window UCB algorithm with a window size $W\in\Re_{+}$ is similar to the UCB algorithm except that the estimated mean rewards are computed by taking the time average of the $W$ \emph{most recent} observed samples.

As noted in Section \ref{sec:intro}, \citep{JakschOA10} proposed the UCRL2 algorithm, which is a UCB-alike algorithm with nearly optimal regret for RL in stationary MDPs. It is thus tempting to think that one could also integrate the forgetting principle into the UCRL2 algorithm to attain low dynamic regret bound for RL in non-stationary MDPs. In particular, one could easily design a naive sliding-window UCRL2 algorithm that follows exactly the same steps as the UCRL2 algorithm with the exception that it uses only the $W$ most recent observed samples instead of all observed samples to estimate the mean rewards and the state transition distributions, and to compute the respective confidence radius.

Under non-stationarity and bandit feedback, however, we show in Proposition \ref{lemma:peril1} of the forthcoming Section \ref{sec:cw_disc} that the diameter of the estimated MDP produced by the naive sliding-window UCRL2 algorithm with window size $W$ can be as large as $\Theta(W)$, which is orders of magnitude larger than $D_{\max}$, the maximum diameter of each individual MDP encountered by the DM. Consequently, the naive sliding-window UCRL2 algorithm may result in undesirable dynamic regret bound. In what follows, we discuss in more details how our novel confidence widening technique can mitigate this issue. 
\subsection{Design Overview}
The \sw~first specifies a sliding window parameter $W\in \mathbb{N}$ and a confidence widening parameter $\eta\geq 0$. Parameter $W$ specifies the number of previous time steps to look at. Parameter $\eta$ quantifies the amount of additional optimistic exploration, on top of the conventional optimistic exploration using upper confidence bounds. The latter turns out to be helpful for handling the temporal drifts in the state transition distributions (see Section \ref{sec:cw_disc}).  

The algorithm runs in a sequence of episodes that partitions the $T$ time steps. Episode $m$ starts at time $\tau(m)$ (in particular $\tau(1) = 1$), and ends at the end of time step $\tau(m + 1) - 1$. Throughout an episode $m,$ the DM follows a certain stationary policy $\tilde{\pi}_{\tau(m)}.$ The DM ceases the $m^{\text{th}}$ episode if at least one of the following two criteria is met: 
\begin{itemize}
	\item The time index $t$ is a multiple of $W.$ Consequently, each episode last for at most $W$ time steps. The criterion ensures that the DM switches the stationary policy $\tilde{\pi}_{\tau(m)}$ frequently enough, in order to adapt to the exogenous dynamics.
	\item There exists some state-action pair $(s, a)$ such that $\nu_{\tau(m)}(s,a),$ the number of time step $t$'s with $(s_t, a_t) = (s, a)$ within episode $m,$ is at least as many as the total number of counts for it within the $W$ time steps prior to $\tau(m),$ \ie, from $(\tau(m)-W)\vee1$ to $(\tau(m)-1).$ This is similar to the doubling criterion in \citep{JakschOA10}, which ensures that each episode is sufficiently long so that the DM can focus on learning. 
\end{itemize} 
The combined effect of these two criteria allows the DM to learn a low dynamic regret policy with historical data from an appropriately sized time window and confidence widening parameter. One important piece of ingredient is the construction of the policy $\tilde{\pi}_{\tau(m)},$ for each episode $m$. To allow learning under endogenous and exogenous dynamics, the \sw~computes the policy $\tilde{\pi}_m$ 
based on the history in the $W$ time steps prior to the current episode $m,$ \ie, from round $(\tau(m) - W) \vee 1$ to round $\tau(m) - 1$. The construction of $\tilde{\pi}_{\tau(m)}$ involves the Extended Value Iteration (EVI) \citep{JakschOA10}, which requires the confidence regions $H_{r, \tau(m)}, H_{p, \tau(m)}(\eta) $ for rewards and state transition distributions as the inputs, in addition to an precision parameter $\epsilon$. The confidence widening parameter $\eta\geq 0$ 
is capable of ensuring the MDP output by the EVI has a bounded diameter most of the time.
\subsection{Policy Construction}
To describe \sw, we first define for each state-action pair $(s, a)$ and each time $t$ in episode $m,$
\begin{equation}\label{eq:N_sa}
N_{t} (s, a) = \sum^{t-1}_{q = (\tau(m) - W) \vee 1} \mathbf{1}((s_q, a_q) = (s, a)), \quad N^+_t  (s, a) = \max\{1, N_{t} (s,a) \}.
\end{equation}
\subsubsection{Confidence Region for Rewards} For each state-action pair $(s,a)$ and each time $t$ in episode $m$, 
we consider the empirical mean estimator
$$\hat{r}_t (s, a) = \frac{1}{N^+_t (s, a)}\left(\sum^{t-1}_{q = (\tau(m) - W) \vee 1 }  R_{q}\left(s , a \right)\mathbf{1}(s_q = s, a_q = a)\right),$$
which serves to estimate the average reward $$\bar{r}_t (s, a)  = \frac{1}{N^+_t(s, a) }\left(\sum^{t-1}_{q = (\tau(m) - W) \vee 1 } r_q(s, a)\mathbf{1}(s_q = s, a_q = a)\right).$$ The confidence region $H_{r, t} = \{H_{r, t}(s, a )\}_{s\in \SSS, a\in \AAA_s}$ is defined as
\begin{equation}\label{eq:Hr}
H_{r, t} (s, a) = \left\{ \dot{r} \in [0, 1] : \left| \dot{r} - \hat{r}_t (s, a) \right| \leq \rad_{r, t}(s, a) \right\},
\end{equation} 
with confidence radius $\rad_{r, t}(s, a) = 2\sqrt{2\log (SA T / \delta ) /N^+_t (s, a)}.$
\begin{algorithm}[!ht]
	\SingleSpacedXI
	\caption{\sw}\label{alg:swr-ucrl2}
	\begin{algorithmic}[1]
		\State \textbf{Input:} Time horizon $T$, state space $\SSS,$ and action space $\AAA,$ window size $W$, confidence widening parameter $\eta.$ 
		\State \textbf{Initialize} $t \leftarrow1,$ initial state $s_1.$
		\For{episode $m = 1, 2, \ldots$}
		\State Set $\tau(m)\leftarrow t$, $\nu_{\tau(m)}(s, a) \leftarrow 0,$ and $N_{\tau(m)}(s, a)$ according to eqn (\ref{eq:N_sa}), for all $s, a$. \; \label{alg:swr-ucrl2-tau}
		\State Compute the confidence regions $H_{r, \tau(m)}$, $H_{p, \tau (m)}(\eta)$ according to eqns (\ref{eq:Hr}, \ref{eq:Hp}).\;
		\State Compute a $(1/\sqrt{\tau(m)})$-optimal optimistic policy $\tilde{\pi}_{\tau(m)}$: \;
		\begin{equation*}
			\textsf{EVI}(H_{r, \tau(m)}, H_{p,\tau( m) }(\eta) ; 1/\sqrt{\tau(m)})\rightarrow(\tilde{\pi}_{\tau(m)},\tilde{r}_{\tau(m)}, \tilde{p}_{\tau(m)}, \tilde{\rho}_{\tau(m)}, \tilde{\gamma}_{\tau(m)}).
		\end{equation*}
		\While{$t$ is not a multiple of $W$ and $\nu_m(s_t, \tilde{\pi}_{\tau(m)}(s_t)) < N^+_{\tau(m)}(s_t, \tilde{\pi}_{\tau(m)}(s_t))$} \label{alg:swr-ucrl2-while}
		\State Choose action $a_t = \tilde{\pi}_{\tau(m)}(s_t),$ observe reward $R_t(s_t, a_t)$ and the next state $s_{t+1}.$ \;\label{alg:swr-ucrl2-action}\label{alg:swr-ucrl2-receive}
		\State Update $\nu_{\tau(m)}(s_t, a_t) \leftarrow \nu_{\tau(m)}(s_t, a_t) + 1,t \leftarrow t+1$.\; \label{alg:swr-ucrl2-end}
		\If{$t> T$}
		\State The algorithm is terminated. \label{alg:swr-ucrl2-break}
		\EndIf
		\EndWhile
		\EndFor
	\end{algorithmic}
\end{algorithm}
\subsubsection{Confidence Widening for State Transition Distributions.} For each state-action pair $s,a$ and each time step $t$ in episode $m$, we consider the empirical mean estimator
$$\hat{p}_t(s'|s, a) =\frac{1}{N^+_t (s, a)}\left(\sum^{t -1}_{q =  (\tau(m) - W) \vee 1 } \mathbf{1}(s_q = s, a_q = a, s_{q + 1} = s')\right),$$ 
which serves to estimate the average transition probability
\begin{align}\label{eq:p_bar}
	\bar{p}_t(s'|s, a) = \frac{1}{N^+_t(s, a)} \sum^{t-1}_{q =  (\tau(m) - W) \vee 1 } p_q(s' | s, a) \mathbf{1}(s_q = s, a_q = a).
\end{align}
Different from the case of estimating reward, the confidence region $H_{p, t}(\eta) = \{H_{p, t} (s, a ; \eta)\}_{s\in \SSS, a\in \AAA_s}$ for the transition probability involves a widening parameter $\eta \geq 0$:
\begin{equation} \label{eq:Hp}
H_{p, t}(s, a; \eta) = \left\{ \dot{p} \in \Delta^\SSS :\left\| \dot{p}(\cdot|s,a) - \hat{p}_t (\cdot | s, a) \right\|_1 \leq \rad_{p,t}(s, a) + \eta\right\},
\end{equation}
with confidence radius $\rad_{p,t}(s, a) = 2\sqrt{2S\log\left(SAT/\delta\right)/N^+_t(s,a)}.$
With $\eta > 0$, the DM can explore state transition distributions that deviate from the sample average, and the exploration is crucial for learning MDPs under endogenous and exogenous dynamics. In a nutshell, the incorporation of $\eta$ provides an additional source of optimism. We treat $\eta$ as a hyper-parameter at the moment, and provide a suitable choice of $\eta$ when we discuss our main results (see Theorem \ref{thm:main1}).

\subsubsection{Extended Value Iteration (EVI) \citep{JakschOA10}.}\label{sec:evi} The \sw~relies on the EVI, which solves MDPs with optimistic exploration to near-optimality. We extract and rephrase a description of EVI in Section \ref{app:evi} of the appendix. EVI inputs the confidence regions $H_r, H_p$ for the rewards and the state transition distributions. The algorithm outputs an ``optimistic MDP model'', which consists of reward vector $\tilde{r}$ and state transition distribution $\tilde{p}$ under which the optimal average gain $\tilde{\rho}$ is the largest among all $\dot{r}\in H_r, \dot{p}\in H_p$:
\begin{itemize}
	\item  \textbf{Input:} Confidence regions $H_r$ for $r$, $H_p$ for $p,$ and an error parameter $\epsilon >0.$ 
	\item \textbf{Output:} The returned policy $\tilde{\pi}$ and the auxiliary output $(\tilde{r}, \tilde{p}, \tilde{\rho}, \tilde{\gamma}).$ In the latter, $\tilde{r},$ $\tilde{p},$ and $\tilde{\rho}$ are the selected ``optimistic" reward vector, state transition distribution, and the corresponding long term average reward. The output $\tilde{\gamma}\in \Re_{+}^{\SSS}$ is a \emph{bias vector} \citep{JakschOA10}. For each $s\in \SSS$, the quantity $\tilde{\gamma}(s)$ is indicative of the short term reward when the DM starts at state $s$ and follows the optimal policy. By the design of EVI, for the output $\tilde{\gamma}$, there exists $s\in \SSS$ such that $\tilde{\gamma}(s) = 0$. Altogether, we express $$\text{EVI}(H_r, H_p ; \epsilon) \rightarrow (\tilde{\pi}, \tilde{r}, \tilde{p}, \tilde{\rho}, \tilde{\gamma}).$$ 
\end{itemize}
Combining the three components, a formal description of the \sw~is shown in Algorithm \ref{alg:swr-ucrl2}.

\subsection{Performance Analysis: The Blessing of More Optimism}
We now analyze the performance of the \sw. First, we introduce two events ${\cal E}_r, {\cal E}_p,$ which state that the estimated reward and state transition distributions lie in the respective (un-widened) confidence regions.
$${\cal E}_r = \{ \bar{r}_t(s, a) \in H_{r, t}(s, a) \ \forall s, a, t  \},\quad{\cal E}_p = \{ \bar{p}_t(\cdot | s, a) \in H_{p, t}(s, a;0) \ \forall s, a, t  \}.$$
We prove that ${\cal E}_r, {\cal E}_p$ hold with high probability.
\begin{lemma}
	\label{lemma:estimation}
	We have $\Pr[{\cal E}_r] \geq 1 - \delta / 2$, $\Pr[{\cal E}_p ] \geq 1 - \delta / 2$.
\end{lemma}
The proof of Lemma \ref{lemma:estimation} is provided in Section \ref{sec:lemma:estimation} of the appendix. In defining ${\cal E}_p$, the widening parameter $\eta$ is set to be 0, since we are only concerned with the estimation error on $p$. Next, we bound the dynamic regret of each time step, under certain assumptions on $H_{p,t}(\eta)$. To facilitate our discussion, we define the following variation measure for each $t$ in an episode $m$: 
$$\dev_{r, t} = \sum^{t-1}_{q = (\tau(m) - W)\vee 1}B_{r, q},\quad \dev_{p, t} =  \sum^{t-1}_{q = (\tau(m) - W)\vee1} B_{p, q}.$$
\begin{proposition}\label{prop:error}
	Consider an episode $m$. Condition on events ${\cal E}_r, {\cal E}_p,$ and suppose that there exists a state transition distribution $p$ satisfying two properties: (1) $\forall s\in\SSS~\forall a\in\AAA_s,$ we have $p(\cdot|s,a)\in H_{p,\tau(m)}(s,a ; \eta)$, and (2) the diameter of $(\SSS,\AAA,p)$ at most $D$. Then, for every $t\in \{\tau(m), \ldots, \tau(m+1) - 1\}$ in episode $m$, we have
	\begin{align}
	\rho^*_t - r_t(s_t, a_t) \leq & \left[ \sum_{s' \in \SSS} p_t(s' | s_t, a_t) \tilde{\gamma}_{\tau(m)}(s') \right] - \tilde{\gamma}_{\tau(m)}(s_t)  \label{eq:prop_error_1}\\
	+ \frac{1}{\sqrt{\tau (m)}} + & \left[2 \dev_{r, t} + 4 D( \dev_{p, t} + \eta)\right] +  \left[ 2\rad_{r,\tau(m)}(s_t, a_t)  + 4 D \cdot \rad_{p,\tau (m)}(s, a) \right] \label{eq:prop_error_2}   .
	\end{align}
\end{proposition}
The proof of Proposition \ref{prop:error} is provided in Section \ref{app:aux_prove_prop_error} of the appendix. Next, we state our first main result, which provides a dynamic regret bound assuming the knowledge of $B_r, B_p$ to set $W, \eta$:
\begin{theorem}\label{thm:main1}
	Assuming $S>1,$ the \sw~with window size $W,$ confidence widening parameter $\eta >0,$ and $\delta=T^{-1}$ satisfies the dynamic regret bound
	\begin{equation*}
	\tilde{O}\left(\frac{B_pW}{\eta}+B_r W +\frac{\sqrt{SA}T}{\sqrt{W}}+ D_\text{max}\left[ B_p W + \frac{S\sqrt{A}T}{\sqrt{W}} + T\eta + \frac{SAT}{W}  + \sqrt{T}\right]\right).	\end{equation*}
	If we further put $W=W^*=S^{2/3}A^{1/2}T^{1/2}(B_r + B_p)^{-1/2}$ and $\eta=\eta^*: = \sqrt{B_pW^*T^{-1}},$ this is $\tilde{O}\left( D_\text{max} (B_r + B_p)^{1/4} S^{2/3}A^{1/2}T^{3/4}\right).$
\end{theorem}
The proof of Theorem \ref{thm:main1} is provided in Section \ref{sec:thm:main1} of the appendix. 
\begin{remark}[\textbf{Confidence Widening}]
	Similar to the regret analysis of the UCRL2 algorithm (Section 4 of \citep{JakschOA10}) and the UCRL2B algorithm (Lemma 3 and eqn. (10) of \citep{FruitPL19}), Proposition \ref{prop:error} states that, if the confidence region $H_{p,\tau(m)}(\eta)$ contains a state transition distribution with diameter at most $D,$ then the EVI provided with $H_{p,\tau(m)}(\eta)$ returns a policy with dynamic regret bound that grows at most linearly with $D$ during episode $m$. However, as shown in Section \ref{sec:cw_disc} later, the parameter $\eta$ has to be carefully chosen for $D$ to be small as the worst case diameter of every state transition distribution in $H_{p,\tau(m)}(0)$ (\ie, setting $\eta=0$) can grow as $\tilde{\Omega}(\sqrt{W}),$ and might result in unfavorable dynamic regret bound. Here, the parameter $\eta$ is the keystone of our novel confidence widening technique and the resulting dynamic regret bound: As $\eta$ increases, the confidence region $H_{p, \tau(m)}(s, a ;\eta)$ becomes larger for each state-action pair $(s, a)$. Consider the first time step $\tau(m)$ of each episode $m:$ if $p_{\tau(m)}(\cdot|s,a)\in H_{p,\tau(m)}(s,a;\eta)$ for all state-action pair $(s,a),$ then Proposition \ref{prop:error} can be leveraged; otherwise, the widened confidence region enforces that a considerable amount of variation budget is consumed. 
\end{remark} 
\begin{remark}[\textbf{Connections with Dynamic Regret Bounds for Non-Stationary MAB}]
	\label{remark:bandit}
	When $S = \{s\}$, our problem becomes the non-stationary bandit problem studied by \citep{BGZ14},  and we have $D_\text{max} = 0$ and $B_p=0.$ By choosing $W=W^*= A^{1/3}T^{2/3} /B_r^{2/3}$, our algorithm has dynamic regret $\tilde{O}(B^{1/3}_r A^{1/3} T^{2/3})$, matching the minimax optimal dynamic regret bound by \citep{BGZ14} when $B_r\in[A^{-1},A^{-1}T].$ 
\end{remark}
\begin{remark}[\textbf{Connections with Regret Bounds for RL in Stationary MDPs}]
	\label{remark:rl}
	When $r_1=\ldots=r_T$ and $p_1=\ldots=p_T,$ our problem becomes the RL in stationary MDPs problem studied by \citep{JakschOA10}, and the \sw~with $W=T$ and $\eta=0$ can recover the regret bound $\tilde{O}(D_{\max}S\sqrt{AT})$ of the UCRL2 algorithm studied in \citep{JakschOA10}.
\end{remark}
\begin{remark}[\textbf{Dynamic Regret Bound without Knowing $B_r$ and $B_p$}]
	\label{remark:obl}
	Similar to \citep{CSLZ19,CSLZ19a}, if $B_p, B_r$ are not known, we can set $W$ and $\eta$ obliviously as $W=S^{\frac{2}{3}}A^{\frac{1}{2}} T^{\frac{1}{2}},\quad\eta= \sqrt{{W}/{T}}=S^{\frac{2}{3}}A^{\frac{1}{2}} T^{-\frac{1}{2}}$ to obtain a dynamic regret bound $\tilde{O}\left( D_\text{max} (B_r + B_p+1) S^{2/3}A^{1/2}T^{3/4}\right).$
\end{remark}

\section{Bandit-over-Reinforcement Learning Algorithm: Towards Parameter-Free}\label{sec:borl}
As pointed out by Remark \ref{remark:obl}, in the case of unknown $B_r$ and $B_{p},$ the dynamic regret of \sw~scales linearly in $B_r$ and $B_p$. However, by Theorem \ref{thm:main1}, we are assured a fixed pair of parameters $(W^*,\eta^*)$ can ensure low dynamic regret. For the bandit setting, \citep{CSLZ19a,CSLZ19} propose the bandit-over-bandit framework that uses a separate copy of EXP3 algorithm to tune the window size. Inspired by it, we develop a novel Bandit-over-Reinforcement Learning (\texttt{BORL}) algorithm with parameter-free $\tilde{O}\left(D_{\max}(B_r+B_p+1)^{1/4}S^{2/3}A^{1/2}T^{3/4}\right)$ dynamic regret here.
\subsection{Design Overview}
\label{sec:bob_intuition}
Following a similar line of reasoning as \citep{CSLZ19a}, we make use of the \sw~as a sub-routine, and ``hedge" \citep{BC12} against the (possibly adversarial) changes of $r_t$'s and $p_t$'s to identify a reasonable fixed window size and confidence widening parameter. 

As illustrated in Fig. \ref{fig:borl_int}, the \borl~divides the whole time horizon into $\lceil T/H\rceil$ blocks of equal length $H$ rounds (the length of the last block can $\leq H$), and specifies a set $J$ from which each pair of (window size, confidence widening) parameter are drawn from. For each block $i\in\left[\lceil T/H\rceil\right]$, the \borl~first calls some master algorithm to select a pair of (window size, confidence widening) parameters $(W_i,\eta_i)~(\in J)$, and restarts the \sw~with the selected parameters as a sub-routine to choose actions for this block. Afterwards, the total reward of block $i$ is fed back to the master, and the ``posterior" of these parameters are updated accordingly. 

One immediate challenge not presented in the bandit setting \citep{CSLZ19} is that the starting state of each block is determined by previous moves of the DM. Hence, the master algorithm is not facing a simple oblivious environment as the case in \citep{CSLZ19}, and we cannot use the EXP3 \citep{ABFS02} algorithm as the master. Nevertheless, fortunately the state is observed before the starting of a block. Thus, we use the EXP3.P algorithm for multi-armed bandit against an adaptive adversary \citep{ABFS02,BC12} as the master algorithm. We follow the exposition in Section 3.2 in \citep{BC12} for adapting the EXP3.P algorithm.
\begin{figure}[!ht]
	\centering
	\includegraphics[width=13.9cm,height=4cm]{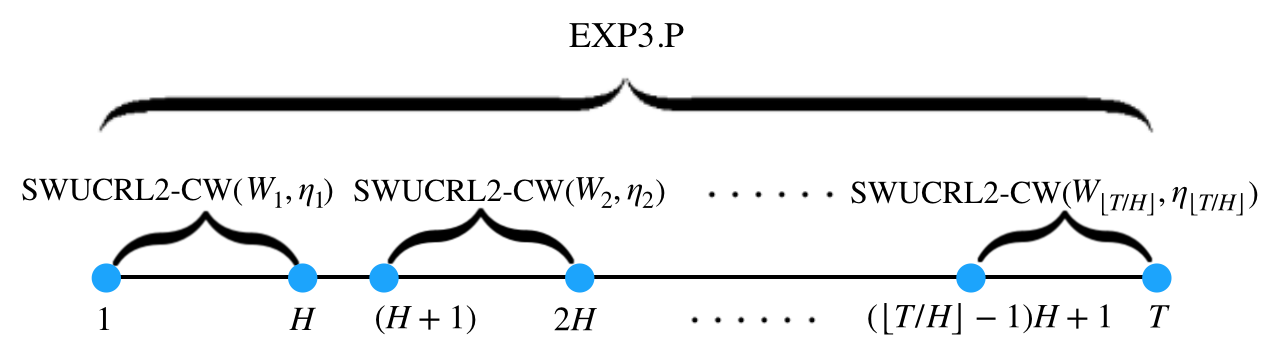}
	\caption{Structure of the \borl}
	\label{fig:borl_int}
\end{figure}

\subsection{Design Details}
We are now ready to state the details of the \borl. For some fixed choice of block length $H$ (to be determined later), we first define a couple of additional notations:
\begin{align}
\label{eq:borl_parameter}&H=\lfloor3S^{\frac{2}{3}}A^{\frac{1}{2}}T^{\frac{1}{2}}\rfloor,\Phi=\frac{1}{2T^{\frac{1}{2}}},\Delta_W=\left\lfloor\ln H\right\rfloor,\Delta_{\eta}=\left\lfloor\ln\Phi^{-1}\right\rfloor,\Delta=(\Delta_W+1)(\Delta_{\eta}+1),\\
\nonumber&J_W=\left\{H^0,\left\lfloor H^{\frac{1}{\Delta_W}}\right\rfloor,\ldots, H\right\}, J_{\eta}=S^{\frac{1}{3}}A^{\frac{1}{4}}\times\left\{\Phi^0,\Phi^{\frac{1}{\Delta_{\eta}}},\ldots, \Phi\right\},
J=\left\{(W,\eta):W\in J_W,\eta\in J_{\eta}\right\}.
\end{align}
Here, $J_W$ and $J_{\eta}$ are all possible choices of window size and confidence widening parameter, respectively, and $J$ is the Cartesian product of them with $|J|=\Delta.$ We also let $\reward_i(W,\eta,s)$ be the total rewards for running the \sw~with window size $W$ and confidence widening parameter $\eta$ for block $i$ starting from state $s,$

The EXP3.P algorithm 
treats each element of $J$ as an arm. It begins by initializing
\begin{align}
\label{eq:borl_exp3_parameter}
&\alpha=0.95\sqrt{\frac{\ln\Delta}{\Delta\lceil T/H\rceil}},\beta=\sqrt{\frac{\ln\Delta}{\Delta\lceil T/H\rceil}},\gamma=1.05\sqrt{\frac{\Delta\ln\Delta }{\lceil T/H\rceil}},q_{(j, k),1}=0~~\forall~(j,k)\in M,
\end{align}
where $M=\{(j',k'):j'\in\{0,1,\ldots,\Delta_W\},k'\in\{0,1,\ldots,\Delta_{\eta}\}\}.$
At the beginning of each block $i\in\left[\lceil T/H\rceil\right],$ the \borl~first sees the state $s_{(i-1)H+1},$ and computes
\begin{align}\label{eq:p}
\forall~(j,k)\in M,\quad u_{(j,k),i}=(1-\gamma)\frac{\exp(\alpha q_{(j,k),i})}{\sum_{(j',k')\in M}\exp(\alpha q_{(j',k'),i})}+\frac{\gamma}{\Delta}.
\end{align}
Then it sets $(j_i,k_i)=(j,k)$ with probability $u_{(j,k),i}~\forall~(j,k)\in M.$ The selected pair of parameters are thus $W_i=\left\lfloor H^{j_i/\Delta_W}\right\rfloor$ and $\eta_i=\left\lfloor \Phi^{k_i/\Delta_{\eta}}\right\rfloor.$ Afterwards, the \borl~starts from state $s_{(i-1)H+1},$ selects actions by running the \sw~with window size $W_i$ and confidence widening parameter $\eta_i$ for each round $t$ in block $i.$ At the end of the block, the \borl~observes the total rewards $\reward\left(W_i,\eta_i,s_{(i-1)H+1}\right).$ As a last step, it rescales  $\reward\left(W_i,\eta_i,s_{(i-1)H+1}\right)$ by dividing it by $H$ so that it is within $[0,1],$ and updates 
\begin{align}
\label{eq:borl_update}
\forall~(j,k)\in M~~~
q_{(j,k),i+1}=q_{(j,k),i}+\frac{\beta+\mathbf{1}_{(j,k)=(j_i,k_i)}\cdot\reward_i\left(W_i,\eta_i,s_{(i-1)H+1}\right)/H}{u_{(j,k),i}}.
\end{align}
The formal description of the \borl~(with $H$ defined in the next subsection) is shown in Algorithm \ref{alg:borl}.
\begin{algorithm}[!ht]
	\SingleSpacedXI
	\caption{\borl}
	\label{alg:borl}
	\begin{algorithmic}[1]
		\State \textbf{Input:} Time horizon $T$, state space $\SSS,$ and action space $\AAA,$ initial state $s_1.$
		\State \textbf{Initialize} $H,\Phi,\Delta_W,\Delta_{\eta},\Delta, J_W,J_\eta$ according to eqn. (\ref{eq:borl_parameter}), and $\alpha,\beta,\gamma$ according to eqn. \eqref{eq:borl_exp3_parameter}.
		\State {$M\leftarrow\{(j',k'):j'\in\{0,1,\ldots,\Delta_W\},k'\in\{0,1,\ldots,\Delta_{\eta}\}\},q_{(j,k),1}\leftarrow 0~\forall (j,k)\in M.$}		
		\For{$i=1,2,\ldots,\lceil T/H\rceil$}
		\State {Define distribution $(u_{(j, k),i})_{(j,k)\in M}$ according to eqn. \eqref{eq:p}, and set $(j_i,k_i)\leftarrow (j,k)$ with probability $u_{(j,k),i}$.}
		\State{$W_i\leftarrow\left\lfloor H^{j_i/\Delta_W}\right\rfloor,\eta_i\leftarrow\left\lfloor \Phi^{k_i/\Delta_{\eta}}\right\rfloor.$}
		\For{$t=(i-1)H+1,\ldots,i\cdot H\wedge T$}
		\State{Run the \sw~with window size $W_i$ and blow up parameter $\eta_i,$ and observe the total rewards $\reward\left(W_i,\eta_i,s_{(i-1)H+1}\right).$}
		\EndFor
		\State{Update $q_{(j,k),i+1}$ according to eqn. (\ref{eq:borl_update}).}
		\EndFor
	\end{algorithmic}
\end{algorithm}
\subsection{Performance Analysis}
\label{sec:borl_analysis}
The dynamic regret guarantee of the \borl~can be presented as follows
\begin{theorem}
	\label{thm:borl}
	Assume $S > 1,$ with probability $1-O(\delta),$ the dynamic regret bound of the \borl~is $\tilde{O}\left(D_{\max}(B_r+B_p+1)^{1/4}S^{2/3}A^{1/2}T^{3/4}\right)$ 
\end{theorem}
The proof of Theorem \ref{thm:borl} is provided in Section \ref{sec:thm:borl} of the appendix.
\section{The Perils of Drift in Learning Markov Decision Processes}\label{sec:cw_disc}
In stochastic online learning problems, one usually estimates a latent quantity by taking the time average of observed samples, even when the sample distribution varies across time. This has been proved to work well in stationary and non-stationary bandit settings \citep{ABF02,GarivierM11,CSLZ19a,CSLZ19}. To extend to RL, it is natural to consider the sample average transition distribution $\hat{p}_t$, which uses the data in the previous $W$ rounds to estimate the time average transition distribution $\bar{p}_t$ to within an additive error $\tilde{O}(1/  \sqrt{N^+_t(s, a}))$ (see Section \ref{sec:evi} and Lemma \ref{lemma:estimation}). In the case of stationary MDPs, where $\forall~t\in[T]~p_t= p$, one has $\bar{p}_t = p.$ Thus, the un-widened confidence region $H_{p,t}(0)$ contains $p$ with high probability (see Lemma \ref{lemma:estimation}). Consequently, the UCRL2 algorithm by \citep{JakschOA10}, which optimistic explores $H_{p,t}(0)$, has a regret that scales linearly with the diameter of $p$.

The approach of optimistic exploring $H_{p, t}(0)$ is further extended to RL in \emph{piecewise-stationary} MDPs by \citep{JakschOA10,GajaneOA18}. The latter establishes a $O(\ell^{1/3} D^{2/3}_{\text{max}} S^{2/3} A^{1/3}T^{2/3})$ dynamic regret bounds, when there are at most $\ell$ changes. Their analyses involve partitioning the $T$-round horizon into $C \cdot T^{1/3}$ equal-length intervals, where $C$ is a constant dependent on $D_\text{max}, S, A, \ell$. At least $C T^{1/3} - \ell$ intervals enjoy stationary environments, and optimistic exploring $H_{p, t}(0)$ in these intervals yields a dynamic regret bound that scales linearly with $D_\text{max}$. Bounding the dynamic regret of the remaining intervals by their lengths and tuning $C$ yield the desired bound.   

In contrast to the stationary and piecewise-stationary settings, optimistic exploration on  $H_{p,t}(0)$ might lead to unfavorable dynamic regret bounds in non-stationary MDPs. In the non-stationary environment where $p_{t-W}, \ldots, p_{t-1}$ are generally distinct, we show that it is impossible to bound the diameter of $\bar{p}_t$ in terms of the maximum of the diameters of $p_{t-W}, \ldots, p_{t-1}$. More generally, we demonstrate the previous claim not only for $\bar{p}_t$, but also for every $\tilde{p}\in H_{p, t}(0)$ in the following Proposition. The Proposition showcases the unique challenge in exploring non-stationary MDPs that is absent in the piecewise-stationary MDPs, and motivates our notion of confidence widening with $\eta > 0$. To ease the notation, we put $t=W+1$ without loss of generality. 
\begin{proposition}\label{lemma:peril1}
	There exists a sequence of non-stationary MDP transition distributions $p_1, \ldots, p_W$ such that 
	\begin{itemize}
		\item The diameter of $(\SSS, \AAA, p_n)$ is $1$ for each $n\in[W].$
		\item The total variations in state transition distributions is $O(1).$
	\end{itemize}
Nevertheless, under some deterministic policy,
\begin{enumerate}
	\item The empirical MDP $(\SSS, \AAA, \hat{p}_{W+1})$ has diameter $\Theta(W)$
	\item Further,  for every $\tilde{p}\in H_{p, W+1}(0),$ the MDP $(\SSS, \AAA, \tilde{p})$ has diameter $\Omega(\sqrt{W/\log W})$
\end{enumerate}
\end{proposition}
\begin{proof}{\textit{Proof.}}
	The sequence  $p_1, \ldots, p_W$ alternates between the following 2 instances $p^1, p^2$. Now, define the common state space $\SSS = \{1, 2\}$ and action collection $\AAA = \{\AAA_1, \AAA_2\}$, where $\AAA_1 = \{a_1, a_2\}$, $\{\AAA_2\} = \{b_1, b_2\}$. We assume all the state transitions are deterministic, and a graphical illustration is presented in Fig. \ref{fig:peril}. Clearly, we see that both instances have diameter $1$. 
	\begin{figure}[!ht]
		\centering
		\includegraphics[width=12.5cm,height=3.45cm]{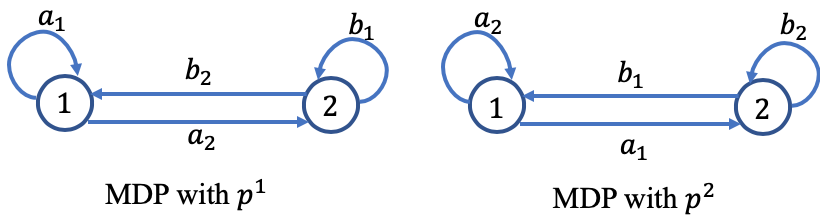}
		\caption{Example MDPs. Since the transitions are deterministic, the probabilities are omitted.}
		\label{fig:peril}
	\end{figure}

	Now, consider the following two deterministic and stationary policies $\pi^1:$ $\pi^1(1) = a_1, \pi^1(2) = b_2,$ and $\pi^2:$ $\pi^2(1) = a_2,\pi^2(2) = b_1.$ Since the MDP is deterministic, we have $\hat{p}_{W+1} = \bar{p}_{W+1}$.
	
	In the following, we construct a trajectory where the DM alternates between policies $\pi^1, \pi^2$ during time  $\{1, \ldots, W\}$  while the underlying transition distribution alternates between $p^1, p^2$. In the construction, the DM is almost always at the self-loop at state 1 (or 2) throughout the horizon, no matter what action $a_1, a_2 $ (or $b_1, b_2$) she takes. Consequently, it will trick the DM into thinking that $\hat{p}_{W+1}(1|1, a_i) \approx 1$ for each $i\in \{1, 2\}$, and likewise  $\hat{p}_{W+1}(2|2, b_i) \approx 1$ for each $i\in \{1, 2\}$. Altogether, this will lead the DM to conclude that $(\SSS, \AAA, \hat{p}_{W+1})$ constitute a high diameter MDP, since the probability of transiting from state 1 to 2 (and 2 to 1) are close to 0.
	
	The construction is detailed as follows. Let $W = 4\tau$. In addition, let the state transition distributions be
 	$$p_1 = \ldots = p_\tau = p^1,\quad p_{\tau + 1} = \ldots = p_{2 \tau} = p^2,\quad p_{2\tau + 1} = \ldots = p_{3 \tau } = p^1,\quad p_{3\tau + 1} = \ldots = p_{4 \tau} = p^2.$$
	The DM starts at state 1. She follows policy $\pi^1$ from time 1 to time $2\tau$, and then policy $\pi^2$ from $2 \tau + 1$ to $4\tau$.

	Under the specified MDP models and policies, it can be readily verified that the DM takes action $a_1$ from time $1$ to $\tau + 1$, action $b_2$ from time $\tau + 2$ to $2\tau$, action $b_1$ from time $2\tau + 1$ to $3\tau + 1$, and action $a_2$ from time $3\tau + 2$ to $4\tau$. As a result, the DM is at state $1$ from time $1$ to $\tau + 1$,  state $2$ from time $\tau + 2$ to $3\tau + 1$,  and eventually state $1$ from time $3\tau + 2$ to $4\tau $ as depicted in Fig. \ref{fig:peril1}. 
	\begin{figure}[!ht]
		\centering
		\includegraphics[width=12cm,height=3.8cm]{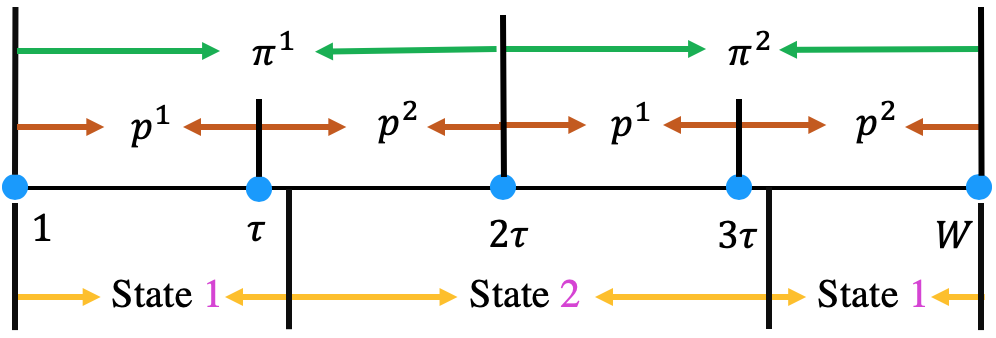}
		\caption{Illustration of the latent MDPs, policies, and state visits.}
		\label{fig:peril1}
	\end{figure}
	We thus have:
	\begin{align*}
	\hat{p}_{W+1}(1 | 1, a_1) =  \frac{\tau}{\tau + 1},\quad \hat{p}_{W+1}(2 | 1, a_1) = \frac{1}{\tau  +1}, \quad\hat{p}_{W+1}(1 | 1, a_2) =1,\quad \hat{p}_{W+1}(2 | 1, a_2) = 0 \\
	\hat{p}_{W+1}(2 | 2, b_1)  = \frac{\tau}{\tau + 1}, \quad \hat{p}_{W+1}(1 | 2, b_1) = \frac{1}{\tau + 1},
	\quad\hat{p}_{W+1}(2 | 2, b_2) =1, \quad\hat{p}_{W+1}(1 | 2, b_2) = 0,
	\end{align*}
	and It can be readily verified that the diameter of $(\SSS,\AAA,\hat{p}_{W+1})$ is $\tau+1=\Theta(W).$ Finally, for the confidence region $H_{p, W+1}(0) = \{H_{p, W+1}(s, a ; 0)\}_{s, a}$ constructed without confidence widening, for any $\tilde{p}\in H_{p, W+1}(0)$ we have 
	$$\tilde{p}(2 | 1, a_1)=\tilde{p}(1 | 2, b_1)= O\left(\sqrt{\frac{\log W}{\tau+1}}\right),\quad\tilde{p}(2 | 1, a_2)=\tilde{p}(1 | 2, b_2) = O\left(\sqrt{\frac{\log W}{\tau-1}}\right)$$
	respectively. Since the stochastic confidence radii $\Theta\left(\sqrt{\frac{\log W}{\tau+1}}\right)$ and $\Theta\left(\sqrt{\frac{\log W}{\tau-1}}\right)$ dominate the sample mean $\frac{1}{\tau + 1}$ and $0$. Therefore, for any $\tilde{p}\in H_{p, W+1}(0)$, the diameter of the MDP constructed by $(\SSS, \AAA, \tilde{p})$ is at least $\Omega\left(\sqrt{\frac{W}{\log W}}\right)$. \halmos
\end{proof}
\begin{remark}
	In Proposition \ref{lemma:peril1}, there are two reasons for the discrepancy between the individual MDPs $p_1, \ldots, p_W$ and the MDPs in the un-widened confidence region $H_{p, W+1}(0)$: 
	\begin{itemize} 
		\item First, due to the bandit feedback, the samples used to construct $\hat{p}_{W+1}$ come from different state-action pairs at different time. As a result, $\hat{p}_{W+1}$ and $\bar{p}_{W+1}$ can be very different than each of the individual state transition probability distributions $p_1, \ldots, p_W$. 
		\item Second, the number of visits to each state-action pair is roughly $W/4,$ which means we would have very ``narrow" confidence regions (of the order $\tilde{O}(1/\sqrt{W})$) if we follow standard optimistic exploration techniques based on concentration inequalities (\ie, the confidence regions shrink as the number of samples grows).
	\end{itemize}
Critically, as shown in Proposition \ref{prop:error} (as well as Section 4 of \citep{JakschOA10}) as well as Lemma 3 and eqn. (10) of \citep{FruitPL19}), the minimum diameter of the MDPs in the confidence regions play a key role in leading to low (dynamic) regret bounds. We thus believe the caveat in learning non-stationary MDPs via conventional optimistic exploration is fundamental in general. In the current paper, we leverage our novel confidence widening technique to prevent the confidence regions from becoming too narrow even if we have lots of samples.
\end{remark}
\begin{remark}
	Inspecting the prevalent OFU guided approach for stochastic MAB and RL in MDPs settings \citep{ABF02,AYPS11,JakschOA10,BC12,LS18}, one usually concludes that a tighter design of confidence region can result in a lower (dynamic) regret bound. In \citep{AbernethyAZ16}, this insights has been formalized in stochastic K-armed bandit settings via a potential function type argument. Nevertheless, Proposition \ref{lemma:peril1} (together with Theorem \ref{thm:main1}) demonstrates that using the tightest confidence region in learning algorithm design may not be enough to ensure low dynamic regret bound for RL in non-stationary MDPs.
\end{remark}
\section{Alternative for Confidence Widening with Application in Inventory Control}\label{sec:alternative}
As demonstrated in previous sections, running the proposed algorithms with the widened confidence regions can help the DM to attain provably low dynamic regret in general RL in non-stationary MDPs. Nevertheless, confidence widening is not always necessary if the state transition distributions bear a special structure. In particular, we consider the following assumption on the state transition distributions $p_1,\ldots,p_T.$
\begin{assumption}\label{ass:alternative}
	There exists a positive quantity (not necessarily known to the DM) $\zeta\in\Re_{+},$ such that for any pair of states $s,s'\in\SSS,$ there is an action $a_{(s,s')}\in\AAA_s$ that satisfies $p_t\left(s'|s,a_{(s,s')}\right)\geq\zeta$ for all $t\in[T].$
\end{assumption}

We can now analyze the dynamic regret bound of the \sw~under Assumption \ref{ass:alternative}. Here, we follow the notations introduced in Section \ref{sec:model} for consistency. In general, Assumption \ref{ass:alternative} ensures that for every time step $t\in[T],$ there exists a state transition distribution $p\in H_{p,t}(0)$ such that the induced diameter of the MDP $(\SSS,\AAA,p)$ is upper bounded by the constant $\bar{D}:=1/\zeta$ with high probability.
\begin{proposition}\label{prop:diameter}
	Under Assumption \ref{ass:alternative} and conditioned on the event $\cal{E}_p$, there exists a state transition distribution $p$ in the confidence region $H_{p,t}(0),$ such that the induced diameter of the MDP $(\SSS,\AAA,p)$ is at most $\bar{D}:={1}/{\zeta}$ for all $t\in[T].$
\end{proposition}
The proof of Proposition \ref{prop:diameter} is provided in Section \ref{sec:prop:diameter} of the appendix. The proposition indicates that the DM can achieve a bounded dynamic regret by implementing the \sw~with $\eta=0.$ To analyze its dynamic regret bound, we provide a variation of
Proposition \ref{prop:error} as follows.
\begin{proposition}\label{prop:error_alternative}
	Consider an episode $m.$ Conditioning on events ${\cal E}_r, {\cal E}_p,$ then for every $t\in\{\tau(m),\ldots,\tau(m+1)-1\}$ in episode $m,$ we have
		\begin{align*}
		 \rho^*_t-r_t(s_t, a_t)  \leq & \left[ \sum_{s' \in \SSS} p_t(s' | s_t, a_t) \tilde{\gamma}_{\tau(m)}(s') \right] - \tilde{\gamma}_{\tau(m)}(s_t)  \\
		+ \frac{1}{\sqrt{\tau (m)}} + & \left[2 \dev_{r, t} + 4 \bar{D}\dev_{p, t} \right] +  \left[ 2\rad_{r,\tau(m)}(s_t, a_t)  + 4 \bar{D} \cdot \rad_{p,\tau (m)}(s, a) \right].
	\end{align*}
\end{proposition}
The proof is similar to that of Proposition \ref{prop:error} with $D_{\tau(m)}$ replaced by $\bar{D}$ and $\eta$ set to $0,$ respectively. We are now ready to state the dynamic regret bound of the \sw~when Assumption \ref{ass:alternative} holds.
\begin{theorem}\label{thm:main_alternative}
	Under Assumption \ref{ass:alternative} and assuming $S>1,$ the \sw~with window size $W,$ confidence widening parameter $\eta=0,$ and $\delta=T^{-1}$ satisfies the dynamic regret bound
	\begin{equation*}
	\text{Dyn-Reg}_T(\texttt{SWUCRL2-CW}) = \tilde{O}\left(B_r W + \bar{D}\left[ B_p W + \frac{S\sqrt{A}T}{\sqrt{W}} + \frac{SAT}{W}  + \sqrt{T}\right]\right)
	\end{equation*}
	If we further put $W=W^*= S^{2/3}A^{1/2} T^{2/3}(B_r + B_p+1)^{-2/3},$ this dynamic regret bound is $\tilde{O}\left( \bar{D}(B_r + B_p+1)^{1/3} S^{2/3}A^{1/2}T^{2/3}\right).$
\end{theorem}
We omit the proof since it is similar to that of Theorem \ref{thm:main1}.
\subsection{An Application to Inventory Control}\label{sec:inv}
In this subsection, we first elaborate on Assumption \ref{ass:alternative} in the context of \emph{single non-perishable item inventory control problem with zero lead time, fixed cost, and lost sales} similar to \citep{YuanLS19}, and then demonstrate how to implement the \sw~for this problem. For each time step $t\in[T]$ of the inventory control problem (with some abuse of notations), the following sequence of events happens:
\begin{enumerate}
	\item The seller first observes her stock level $s_t$, and decides the quantity $a_t$ to order.
	\item If $a_t>0,$ a fixed cost $f$ and a $c$ per-unit ordering cost are incurred, and the order arrives instantaneously. The stock level then becomes $s_t+a_t.$
	\item The demand $X_t$ is realized, and the seller observes the censored demand $Y_t=\min\{X_t,s_t+a_t\}$. The DM faces non-stationary demands, in the sense that the demand distributions $X_1, \ldots, X_T$ at time steps $1, \ldots T$ are independent but not identically distributed.
	\item Unfulfilled demand incurs a $l$ per-unit lost sales cost, while excess inventory leads to a $h$ per-nit holding cost. The total cost for time step $t$ is
	\begin{align}
	C_t(s_t,a_t)=f\cdot\bm{1}[a_t>0]+c\cdot a_t+l\cdot[X_t-s_t-a_t]^++h\cdot[s_t+a_t-X_t]^+.
	\end{align}
	Due to demand censoring, the cost is not \emph{observable}.
\end{enumerate}  
The seller's objective is to minimize the cumulative total cost $\sum_{t=1}^TC_t(s_t,a_t).$ To map this into the non-stationary MDP model we described in Section \ref{sec:model}, we represent the level of stock at the beginning of each time step as the state. Same as \citep{YuanLS19} (and similar to \citep{HuhR09,ZhangCS18,AgrawalJ19}), we assume the DM has a limited shelf capacity, and she can hold at most $S$ units of inventory at any time. Consequently, $\SSS = \{0, \ldots, S\}$, and $\AAA_s = \{0, \ldots, S - s\}$ for each $s\in \SSS$. We also define the reward and state transition distributions for all $t\in[T],$ $s,s'\in\SSS,$ and $a\in\AAA_s$ as follows, $$R_t(s,a)=-C_t(s,a)\qquad\text{ and }\qquad p_t(s'|s,a)=\Pr\left(s+a-\min\{s+a,X_t\}=s'\right).$$  
However, it is worth emphasizing that, different than our setup in Section \ref{sec:model}, $R_t(s,a)$ is not observable as $C_t(s,a)$ is not observable. Nevertheless, we shall demonstrate in Section \ref{sec:inv_impl} that one could use the technique of pseudo-reward proposed in \citep{AgrawalJ19} to bypass this issue.

Following Assumption \ref{ass:alternative}, we make the \emph{strictly positive probability mass function (PMF) assumption} on $X_1, \ldots, X_T.$
\begin{assumption}[Strictly Positive PMF]\label{ass:pmf}
	 There is a $\zeta > 0$ such that $\Pr(X_t = s)\geq \zeta >0$ for all $t\in [T]$ and $s\in \{0, \ldots, S\}$.
\end{assumption} 
\begin{remark}
	It can be readily verified that if the demands satisfy the strictly positive PMF assumption, the underlying inventory control problem satisfies Assumption \ref{ass:alternative}. Indeed, the DM could transit from a state $s\in \SSS$ to another state $s'\in \SSS$ with probability at least $\zeta$ by ordering $S-s$ units of the item, since then $p_t\left(s'|s,S-s\right) = \Pr(X_t = S - s')\geq\zeta$. 
\end{remark}
\subsubsection{Comparisons to Existing Inventory Control Models}
We first compare our setting and existing ones on single non-perishable item inventory control problem with lost sales.

Similar to \citep{HuhR09,ZhangCS18,YuanLS19,AgrawalJ19}, the model presented in this section studies the single non-perishable item inventory control problem with lost sales. However, there are several key differences between ours and the existing works in terms of cost functions, demand distributions, and lead time:
		\begin{table}[!ht]
		\begin{center}
			\begin{tabular}{ | m{5cm} | m{4cm}|m{4.1cm} |  m{2cm}|} 
				\hline
				&Cost functions & Demand distributions & Lead time\\ 
				\hline
				\cite{HuhR09} &linear purchasing cost without fixed cost, linear lost sales and holding cost&  stationary, continuous or discrete  & zero \\ 
				\hline
				\cite{ZhangCS18,AgrawalJ19}&no purchasing cost, linear lost sales and holding cost& stationary, continuous or discrete&  positive \\ 
				\hline
				\cite{YuanLS19}& linear purchasing cost with fixed cost, linear lost sales and holding cost& stationary, continuous or discrete &zeros \\ 
				\hline
				Ours&linear purchasing cost with fixed cost, linear lost sales and holding cost&non-stationary, discrete, with strictly positive PMF&zero\\
				\hline
			\end{tabular}
			\caption{Comparisons between our inventory control model and existing works'}
			\label{table:inventory_model}
		\end{center}
	\end{table}
	\begin{itemize}
		\item\textbf{Cost Functions:}  In \citep{HuhR09}, the authors assume a linear purchasing cost function without fixed cost, linear lost sales and holding cost functions. In \citep{YuanLS19}, the authors additionally allow fixed cost. In \citep{ZhangCS18,AgrawalJ19}, the authors assume the lost sales cost function and the holding cost function are linear, and there is no purchasing cost. In our setting, our cost function is the same as that of \citep{YuanLS19}.  
		\item\textbf{Demand Distributions:} In \citep{HuhR09,ZhangCS18,YuanLS19,AgrawalJ19}, the authors assume stationary demand distributions, but they admit both continuous or discrete demand distributions. In contrast, we allow non-stationary demand distributions, but we impose that the demand distribution has to be discrete, and satisfies the strictly positive PMF assumption described above. 
		\item\textbf{Lead Time:} In \citep{ZhangCS18,AgrawalJ19}, the authors allow the lead time to be positive; while in \citep{HuhR09,YuanLS19} and our setting, we assume the lead time is zero.
	\end{itemize}
	A summary of the comparisons is provided in Table \ref{table:inventory_model}.
\subsubsection{Implementation of the \sw}\label{sec:inv_impl}
As pointed out in Section \ref{sec:inv}, different than the model we present in Section \ref{sec:model}, the reward in each time step $t$ is not directly observable due to the censored demand. Nevertheless, we can follow the pseudo-reward technique proposed in \citep{AgrawalJ19} to implement the \sw~on a sequence of suitably designed pseudo-reward distributions. 

In particular, we define the pseudo-reward following \citep{AgrawalJ19} for each time step $t\in[T],$ every state $s,$ and every action $a\in\AAA_s$ as
\begin{align*}
	R^{\pse}_t(s,a):=R_t(s,a)+l\cdot X_t=-f\cdot\bm{1}[a>0]-c\cdot a_t-h\cdot[s+a-Y_t]^++l\cdot Y_t,
\end{align*}
where we recall $Y_t=\min\{s+a,X_t\}$ is the censored demand. We note that the pseudo-reward is perfectly \emph{observable}. We also define the mean pseudo-reward or each time step $t\in[T],$ every state $s,$ and every action $a\in\AAA_s$ as
\begin{align}\label{eq:pse1}
	r_t^{\pse}(s,a):=\E\left[R^{\pse}_t(s,a)\right]=\E\left[R_t(s,a)+l\cdot X_t\right]=r_t(s,a)+l\cdot\E[X_t].
\end{align}
This indicates regardless of state and action, the mean pseudo-reward of a time step $t$ can be obtained from shifting the corresponding mean reward uniformly by $l\cdot\E[X_t].$ Without loss of generality, we assume for all $t\in[T],$ $s\in\SSS,$ and $a\in\AAA_s,$ the mean pseudo-reward is bounded, \ie, $r^{\pse}_t(s,a)\in[0,1],$ and the pseudo-reward $R^{\pse}_t(s,a)$ is 1-sub-Gaussian with mean $r_t^{\pse}(s,a).$ Defining $\rho^{*\pse}_t$ as the optimal long-term average reward of the stationary MDP with state transition distribution $p_t$ and mean reward $r^{\pse}_t=\{r^{\pse}_t(s,a)\}_{s\in\SSS,a\in\AAA_s},$ we can show that for any policy $\Pi,$ the dynamic regret of the non-stationary MDP instance specified by the tuple $\mathcal{M}=(\SSS,\AAA,T,r,p)$ and the dynamic regret of the non-stationary MDP instance specified by the tuple $\mathcal{M}^{\pse}=(\SSS,\AAA,T,r^{\pse}=\{r^{\pse}_t\}_{t=1}^T,p)$ are the same.
\begin{proposition}\label{prop:pse}
	For any policy $\Pi,$ we denote the sample path for following $\Pi$ on $\mathcal{M}$ as $\{s_t(\mathcal{M}), a_t(\mathcal{M})\}_{t=1}^T,$ and the sample path for following $\Pi$ on $\mathcal{M}^{\pse}$ as $\{s_t(\mathcal{M}^{\pse}), a_t(\mathcal{M}^{\pse})\}_{t=1}^T,$ we have
	$$\sum^T_{t = 1}\left\{\rho^*_t - \mathbb{E}[r_t(s_t(\mathcal{M}), a_t(\mathcal{M})) ] \right\}=\sum^T_{t = 1}\left\{\rho^{*\pse}_t - \mathbb{E}[r^{\pse}_t(s_t(\mathcal{M}^{\pse}), a_t(\mathcal{M}^{\pse})) ] \right\}.$$
\end{proposition}
The proof of Proposition \ref{prop:pse} is provided in Section \ref{sec:prop:pse} in the appendix. Together with Theorem \ref{thm:main_alternative}, we have the following dynamic regret bound guarantee for the \sw~on the the single non-perishable item inventory control problem with zero lead time, fixed cost, and lost sales.
\begin{theorem}\label{thm:main_inventory}
	For the inventory control model in Section \ref{sec:inv}, under Assumption \ref{ass:pmf} and assuming $S>1,$ the \sw~with window size $W,$ confidence widening parameter $\eta=0,$ and $\delta=T^{-1}$ satisfies the dynamic regret bound
	\begin{equation*}
		\text{Dyn-Reg}_T(\texttt{SWUCRL2-CW}) = \tilde{O}\left(B_r W + \bar{D}\left[ B_p W + \frac{S^{\frac{3}{2}}T}{\sqrt{W}} + \frac{S^2T}{W}  + \sqrt{T}\right]\right)
	\end{equation*}
	If we further put $W=W^*= ST^{2/3}(B_r + B_p+1)^{-2/3},$ this dynamic regret bound is $\tilde{O}\left( \bar{D}(B_r + B_p+1)^{1/3} ST^{2/3}\right).$
\end{theorem}
\begin{remark}
To interpret the dynamic regret bound of the \sw~in the context of inventory control, we note that in Theorem \ref{thm:main_inventory}, we normalize the cost functions so that the cost incurs in each time period is in $[0, 1].$ This is slightly different than the setups in \citep{HuhR09,ZhangCS18,YuanLS19,AgrawalJ19}, where the upper bound of the cost functions are of order $O(S).$ 
\end{remark}
\section{Numerical Experiments}\label{sec:numerical}
As a complement to our theoretical results, we conduct numerical experiments on synthetic datasets to compare the dynamic regret performances of our algorithms with the UCRL2 algorithm \citep{JakschOA10}, which is one of the most widely used benchmarks for RL in MDPs due to its nearly-optimal regret bound in stationary environments \citep{WeiJLSJ19}, and also the restarting UCRL2 (denoted as UCRL2.S) algorithm for RL in piecewise-stationary MDPs \citep{JakschOA10}

\noindent\textbf{Setup:} We consider a MDP with 2 states $\{s_1,s_2\}$ and 2 actions $\{a_1,a_2\}$, and set $T=5000.$ The rewards are deterministically set to
\begin{align*}
&r_t(s_1,a_1)=0.2+3\cos\left({5V_r\pi t}/{T}\right),\quad r_t(s_1,a_2)=0.2+\cos\left({5V_r\pi t}/{T}\right),\\
&r_t(s_2,a_1)=0.2-\cos\left({5V_r\pi t}/{T}\right),\quad r_t(s_2,a_2)=0.2-3\cos\left({5V_r\pi t}/{T}\right).
\end{align*}
The total variations in mean rewards is thus $B_r=15V_r=\Theta(V_r).$ An illustration of the reward process of state $s_2$ and action $a_2$ is provided in Fig. \ref{fig:sim_reward} (the mean rewards of other (state,action) pairs are similar).
\begin{figure}[!ht]
	\centering
	\includegraphics[width=8cm,height=6cm]{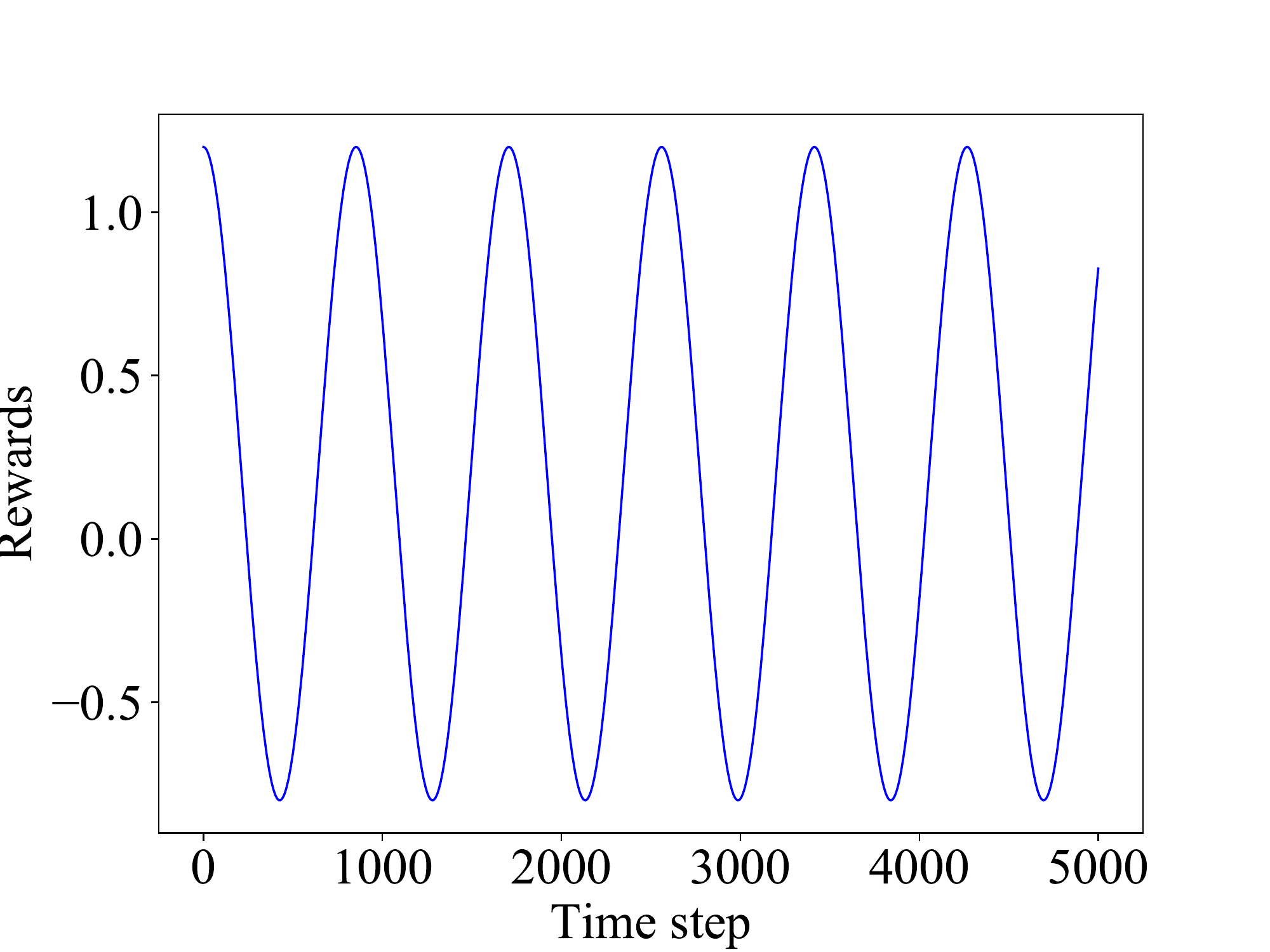}
	\caption{Illustrations of mean rewards $r_t(s_2,a_2)$ (the mean rewards of other state-action pairs are similar)}
	\label{fig:sim_reward}
\end{figure}
The state transition distributions are set to
\begin{align*}
&p_t(s_1|s_1,a_1)=1, \quad p_t(s_2|s_1,a_1)=0, \quad p_t(s_1|s_1,a_2)=1-\beta_t,\quad p_t(s_1|s_1,a_2)=\beta_t,\\
&p_t(s_1|s_2,a_1)=0, \quad p_t(s_2|s_2,a_1)=1, \quad p_t(s_1|s_2,a_2)=\beta_t,\quad p_t(s_1|s_2,a_2)=1-\beta_t.
\end{align*}
where $\beta_t$ is governed by the following process:
\begin{align*}
\beta_t=0.5+0.3\sin\left({5V_p\pi t}/{T}\right).
\end{align*}
The total variations in the state transition distributions is thus $B_p=12V_p=\Theta(V_p).$ In this simulation, we allow both $V_r$ and $V_p$ to take values from $\{T^{0.2},T^{0.5}\}$ to evaluate the performances of the algorithms in low and high variations scenarios. Here, we assume the \sw~knows the variation budgets, and the UCRL2.S algorithm restarts the UCRL2 algorithm every $\lfloor T^{2/3}\rfloor$ time steps. All the results are averaged over 50 runs.
\begin{figure}[!ht]
	\subfigure[$B_p=\Theta(T^{0.2}),\quad B_r=\Theta(T^{0.2})$]{	\label{fig:plrl}\includegraphics[width=8.8cm,height=6.2cm]{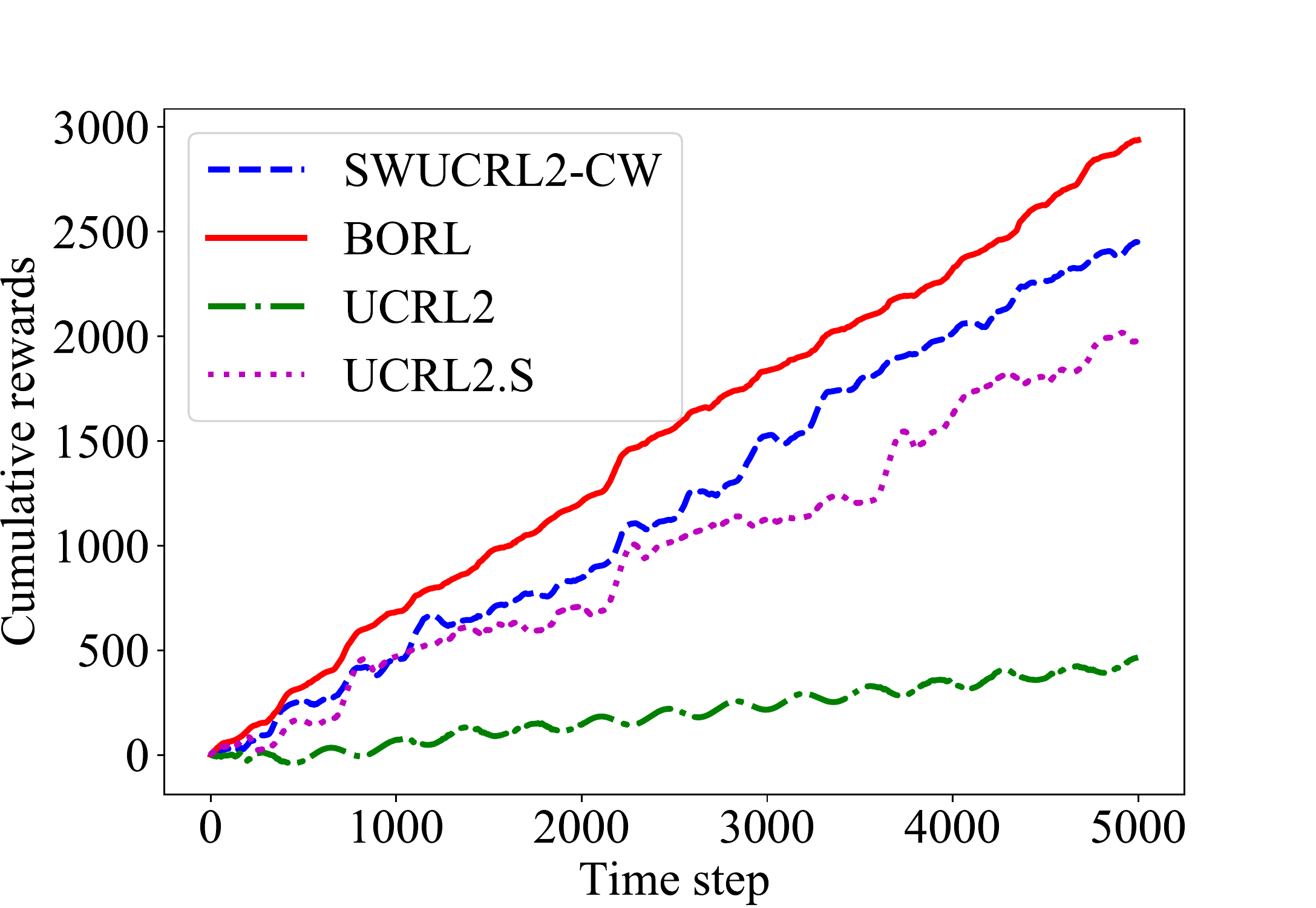}}
	\subfigure[$B_p=\Theta(T^{0.2}),\quad B_r=\Theta(T^{0.5})$]{	\label{fig:plrh}\includegraphics[width=8.8cm,height=6.2cm]{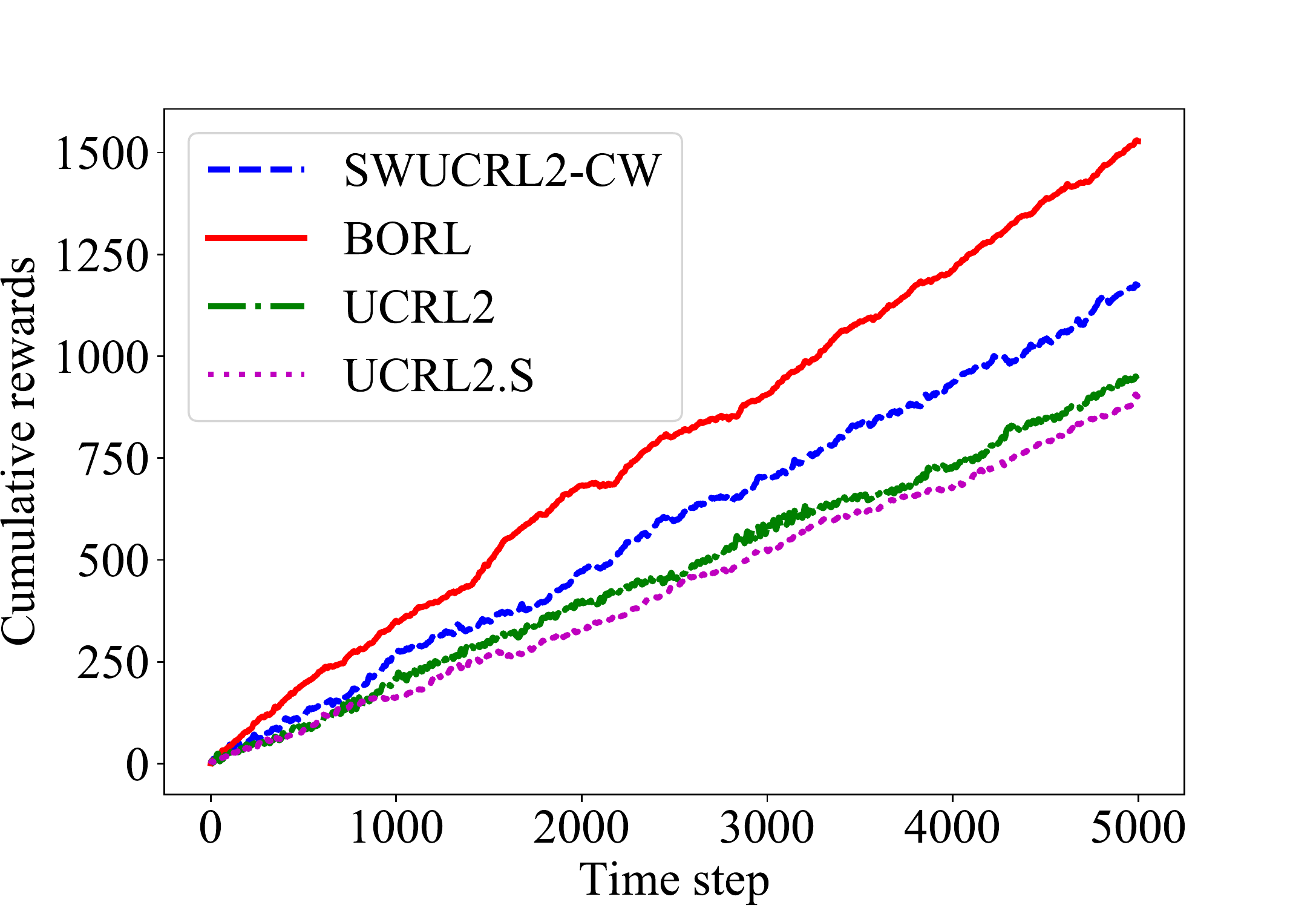}}
	\subfigure[$B_p=\Theta(T^{0.5}),\quad B_r=\Theta(T^{0.2})$]{	\label{fig:phrl}\includegraphics[width=8.8cm,height=6.2cm]{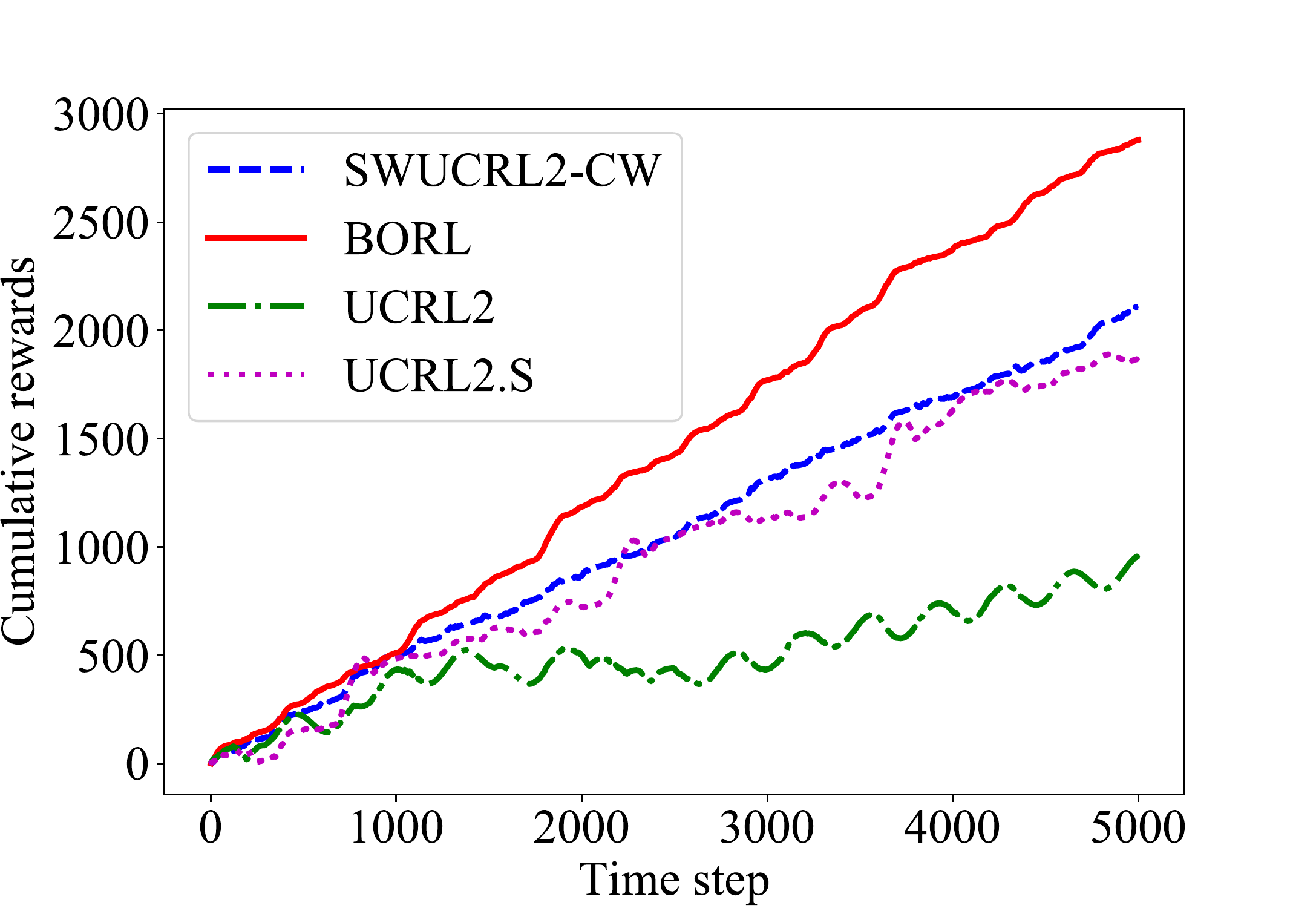}}
	\subfigure[$B_p=\Theta(T^{0.5}),\quad B_r=\Theta(T^{0.5})$]{	\label{fig:phrh}\includegraphics[width=8.8cm,height=6.2cm]{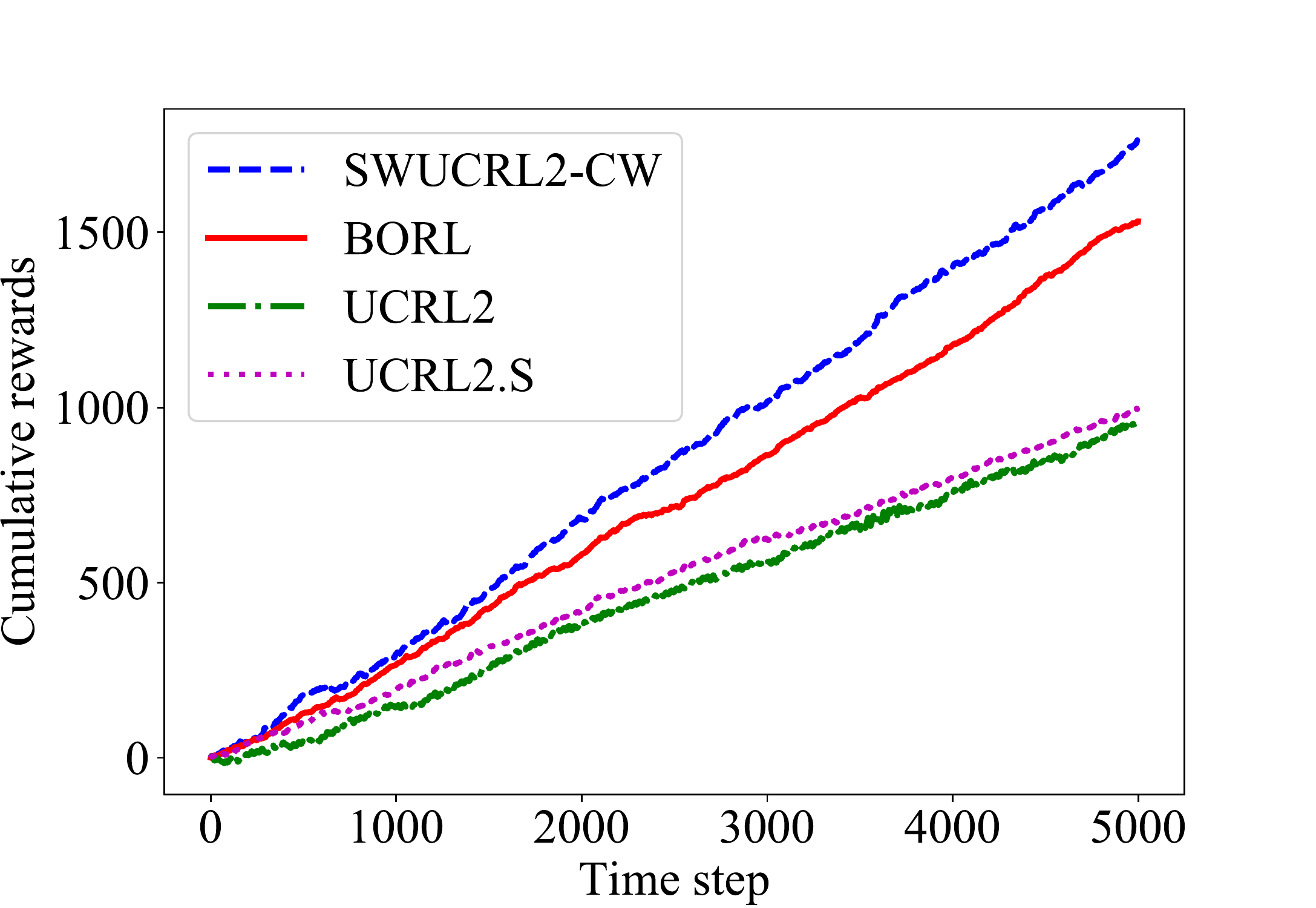}}
	\caption{Cumulative rewards of the algorithms}
	\label{fig:sim_result}
\end{figure}

\noindent\textbf{Results:} The cumulative rewards of the algorithms under various variation budgets are shown in Fig. \ref{fig:sim_result}. The results show that both the \sw~and the \borl~are able to collect at least 20\% more rewards then the UCRL2 algorithm and the UCRL2.S algorithm except for the case when $B_p=\Theta(T^{0.5})$ and $B_r=\Theta(T^{0.2}),$ the percentage improvement is $12\%.$ Comparing the results in Figs. \ref{fig:plrl}, \ref{fig:plrh}, and \ref{fig:phrl}, we can see that both the \sw~and the \borl~are more robust to variations in the state transition distributions than that in reward distributions. This demonstrate the power of our confidence widening technique. Interestingly, we can see that in Figs. \ref{fig:plrl}, \ref{fig:plrh}, and \ref{fig:phrl}, the cumulative rewards of the \borl~(does not know the variation budgets) are higher than those of the \sw~(knows the variation budgets). This indeed has no contradiction to our theoretical results. Theorems \ref{thm:main1} and \ref{thm:borl} state that the \sw~and the \borl~enjoy the same (in the sense of $\tilde{O}(\cdot)$) worst case dynamic regret bound. Nevertheless, the environments we construct in Fig. \ref{fig:sim_reward} are not the worst case scenario, and the results indicate that the adaptive master algorithm (\ie, the EXP3.P algorithm) of the \borl~is able to leverage this more benign environment to attain higher rewards.

\section{Conclusion}\label{sec:conclusion}
In this paper, we study the problem of un-discounted reinforcement learning in a gradually changing environment. In this setting, the parameters, \ie, the reward and state transition distributions, can be different from time to time as long as the total changes are bounded by some variation budgets, respectively. We first incorporate the sliding window estimator and the novel confidence widening technique into the UCRL2 algorithm to propose a \sw~with low dynamic regret when the variation budgets are known. We then design a parameter-free \borl~that allows us to enjoy the same dynamic regret bound as the \sw~without knowing the variation budgets. The main ingredient of the proposed algorithms is the novel confidence widening technique, which injects extra optimism into the design of learning algorithms, and thus ensure low dynamic regret bounds. This is in contrast to the widely held believe that optimistic exploration algorithms for (stationary and non-stationary) stochastic online learning settings should employ the lowest possible level of optimism. To extend this finding, we also use the problem of single-item inventory control with fixed cost as an example to demonstrate how one can leverage special structures in the state transition distributions to attain low dynamic regret bound without widening the confidence region.
\ACKNOWLEDGMENT{The authors would like to express sincere gratitude to Dylan Foster, Negin Golrezaei, and Mengdi Wang, as well as various seminar attendees for helpful discussions and comments.}


\bibliographystyle{ormsv080}
\bibliography{nonstatmdp_ref}
\newpage
\begin{APPENDIX}{Supplementary}
	\section{Supplementary Details about MDPs}
\label{sec:mdp_sup}
\subsection{Linear Program Formulations}\label{sec:mdp_lp}
The optimal long term reward $\rho^*_t$ is equal to the optimal value of the linear program $\textsf{P}(r_t, p_t)$ \citep{Puterman94}. For a reward vector $r$ and a transition distribution $p$, we define 
\begin{align}
\label{eq:primal}
\textsf{P}(r, p): \max    & \sum_{s\in \SSS, a\in \AAA_s}r(s, a) x(s, a) \\
\text{s.t. }   &\sum_{a \in \AAA_s} x(s, a) = \sum_{s'\in\SSS, a'\in \AAA_{s'}} p(s | s',a')x(s', a')      &\quad &\forall s \in \SSS \nonumber\\
&\sum_{s\in \SSS, a\in \AAA_s}x(s, a) = 1       &\quad & \nonumber\\
&x(s, a)\geq 0      &\quad &\forall s\in \SSS, a\in \AAA_s \nonumber
\end{align}
Throughout our analysis, it is useful to consider the following dual formulation $\textsf{D}(r, p)$ of the optimization problem $\textsf{P}(r, p)$:
\begin{align}
\label{eq:dual}
\textsf{D}(r, p): \min   &\quad\rho  \\
\text{s.t.}   &\quad\rho + \gamma(s) \geq r (s, a) + \sum_{s'\in\SSS} p (s' | s,a) \gamma(s')      &\forall s \in \SSS, a\in \AAA_s \nonumber\\
&\quad\phi, \gamma(s)\text{ free}     & \forall s\in \SSS.\quad~\qquad \nonumber
\end{align}
The following Lemma shows that  any feasible solution to $\textsf{D}(r, p)$ is essentially bounded if the underlying MDP is communicating, which will be crucial in the subsequent analysis.
\begin{lemma}\label{lemma:mc}
	Let $(\rho,\gamma)$ be a feasible solution to the dual problem $\textsf{D}(r, p)$, where $(\SSS, \AAA, p)$ consititute a communicating MDP with diameter $D$. We have $$\max_{s, s'\in \mathcal{S}} \left\{ \gamma(s) - \gamma(s') \right\} \leq 2 D.$$ 
\end{lemma}
The Lemma is extracted from Section 4.3.1 of \citep{JakschOA10}, and it is more general than \citep{LS18}, which requires $(\rho, \gamma)$ to be optimal instead of just feasible.

\subsection{Proof of Proposition \ref{prop:benchmark}}\label{app:pfprofbenchmark}
We begin with invoking Lemma \ref{lemma:mc}, which guarantees that for each $t$ there is an optimal solution $(\rho^*_t, \gamma^*_t)$ of $\textsf{D}(r_t, p_t)$ that satisfies $0 \leq \gamma^*_t(s)\leq 2 D_\text{max}$ for all $s\in \SSS$. Recall for each $t$:
\begin{align}
B_{r, t} =\max_{s\in \SSS, a\in \AAA_s}\left| r_{t + 1}(s, a) - r_t(s, a)  \right|,\quad B_{p, t}=\max_{s\in \SSS, a\in \AAA_s }\left\| p_{t + 1}(\cdot | s, a) - p_t ( \cdot  |s, a)  \right\|_1.
\end{align}
Consider two time indexes $t \leq \tau$. We first claim the following two inequalities:
\begin{align}
\rho^*_\tau & \geq \rho^*_t - \sum^{\tau - 1}_{q = t}\left( B_{r, q} + 2 D_\text{max} B_{p, q}\right)\label{eq:benchmark_ineq_1}\\
\rho^*_t &\geq r_\tau (s_\tau, a_\tau)  + \left[ \sum_{s'\in\SSS} p_\tau (s' | s_\tau,a_\tau) \gamma^*_t(s') - \gamma^*_t(s_\tau) \right] - \sum^{\tau - 1}_{q = t}\left( B_{r, q} + 2 D_\text{max} B_{p, q}\right) \label{eq:benchmark_ineq_2}.
\end{align}
The proofs of inequalities (\ref{eq:benchmark_ineq_1}, \ref{eq:benchmark_ineq_2}) are deferred to the end. Now, combining (\ref{eq:benchmark_ineq_1}, \ref{eq:benchmark_ineq_2}) gives
\begin{equation}\label{eq:benchmark_key}
\rho^*_\tau \geq r_\tau (s_\tau, a_\tau)  + \left[ \sum_{s'\in\SSS} p_\tau (s' | s_\tau,a_\tau) \gamma^*_t(s') - \gamma^*_t(s_\tau) \right] - 2\sum^{\tau - 1}_{q = t}\left( B_{r, q} + 2 D_\text{max} B_{p, q}\right).
\end{equation}
Let positive integer $W\leq T$ be a window size, which is specified later. Summing (\ref{eq:benchmark_key}) over $\tau = t, \ldots, t + W - 1$ and taking expectation over $\{(s_\tau, a_\tau)\}^{t+W - 1}_{\tau = t}$ yield
\begin{align}
&\sum^{t-W + 1}_{\tau = t} \rho^*_\tau \geq \mathbb{E}\left[\sum^{t-W + 1}_{\tau = t}  r_\tau (s_\tau, a_\tau) \right] + \mathbb{E}\left[\sum^{t-W }_{\tau = t} p_\tau (s' | s_\tau,a_\tau) \gamma^*_t(s') - \gamma^*_t(s_{\tau+1}) \right] \label{eq:benchmark_step1}\\
+ & \mathbb{E}\left[\sum_{s'\in\SSS} p_{t - W + 1} (s' | s_{t-W + 1},a_{t- W+1}) \gamma^*_t(s')  - \gamma^*_t(s_t)\right] - 2\sum^{t-W + 1}_{\tau = t}\sum^{\tau - 1}_{q = t}\left( B_{r, q} + 2 D_\text{max} B_{p, q}\right)\label{eq:benchmark_step2}\\
&\geq \mathbb{E}\left[\sum^{t-W + 1}_{\tau = t}  r_\tau (s_\tau, a_\tau) \right] -  2 D_\text{max} - 2W\sum^{t + W - 1}_{q = t}\left( B_{r, q} + 2 D_\text{max} B_{p, q}\right)\label{eq:benchmark_key2}.
\end{align}
To arrive at (\ref{eq:benchmark_key2}), note that the second expectation in (\ref{eq:benchmark_step1}), which is a telescoping sum, is equal to 0, since $s_{\tau + 1}$ is distributed as $p( \cdot | s_\tau, a_\tau)$. In addition, we trivially lower bound the first expectation in (\ref{eq:benchmark_step2}) by $-2 D_\text{max}$ by applying Lemma \ref{lemma:mc}. Next, consider partitioning the horizon of $T$ steps into intervals of $W$ time steps, where last interval could have less than $W$ time steps. That is, the first interval is $\{1, \ldots, W\}$, the second is $\{W+1, \ldots, 2W\}$, and so on. Applying the bound (\ref{eq:benchmark_key2}) on each interval and summing the resulting bounds together give
\begin{align}
\sum^T_{t=1} \rho^*_t &\geq \mathbb{E}\left[\sum^T_{t=1} r_t(s_t, a_t)\right] - 2 \lceil \frac{T}{W}\rceil D_\text{max} - 2W\sum^T_{t=1}(B_{r, t} + 2 D_\text{max}B_{p, t})\nonumber\\
&\geq \mathbb{E}\left[\sum^T_{t=1} r_t(s_t, a_t)\right] - \frac{4T D_\text{max}}{W} - 2W (B_r + 2 D_\text{max} B_p).
\end{align}
Choosing $W$ to be any integer in $[\sqrt{T/ (B_r + 2 D_\text{max}B_p)}, 2\sqrt{T/ (B_r + 2 D_\text{max}B_p)}]$ yields the desired inequality in the Theorem. Finally, we go back to proving inequalities (\ref{eq:benchmark_ineq_1},\ref{eq:benchmark_ineq_2}). These inequalities are clearly true when $t = \tau$, so we focus on the case $t < \tau$.

\textbf{Proving inequality (\ref{eq:benchmark_ineq_1}). }It suffices to show that the solution $(\rho^*_\tau + \sum^{\tau - 1}_{q = t}(B_{r, q} + 2 D_\text{max} B_{p, q}), \gamma^*_\tau)$ is feasible to the linear program \textsf{D}$(r_t, p_t)$. To see the feasibility, it suffices to check the constraint of \textsf{D}$(r_t, p_t)$ for each state-action pair $s, a$:
\begin{align*}
& \rho^*_\tau + \sum^{\tau - 1}_{q = t}(B_{r, q} + 2 D_\text{max} B_{p, q})\nonumber\\
\geq & \left[r_\tau(s, a) + \sum^{\tau - 1}_{q = t}B_{r, q}\right] + \left[ - \gamma^*_\tau(s)+ \sum_{s'\in \SSS} p_\tau(s' | s, a)\gamma^*_\tau(s') +\sum^{\tau - 1}_{q = t} 2 D_\text{max} B_{p, q}\right] \nonumber. 
\end{align*}
The feasibility is proved by noting that 
\begin{align}
\left|r_\tau(s, a) - r_t(s,a)\right|& \leq \sum^{\tau - 1}_{q= t} B_{r, q},\label{eq:benchmark_for_r}\\
\left|\sum_{s'\in \SSS} p_\tau(s' | s, a)\gamma^*_\tau(s') - \sum_{s'\in \SSS} p_t(s' | s, a)\gamma^*_\tau(s') \right| &\leq \left\| p_\tau(\cdot | s, a) - p_t(\cdot | s, a)\right\|_1 \left\| \gamma^*_\tau \right\|_\infty \nonumber\\
&\leq \sum^{\tau - 1}_{q = t} B_{p, q} (2 D_\text{max}).\label{eq:benchmark_for_p}
\end{align}

\textbf{Proving inequality (\ref{eq:benchmark_ineq_2}). }We have 
\begin{align}
\rho^*_t &\geq r_t (s_\tau, a_\tau) + \sum_{s'\in\SSS} p_t (s' | s_\tau,a_\tau) \gamma^*_t(s') - \gamma^*_t(s_\tau) \nonumber\\
&\geq r_\tau (s_\tau, a_\tau) + \sum_{s'\in\SSS} p_t (s' | s_\tau,a_\tau) \gamma^*_t(s') - \gamma^*_t(s_\tau)  - \sum^{\tau - 1}_{s = t} B_{r, s}  \label{eq:benchmark_by_r}\\
&\geq r_\tau (s_\tau, a_\tau)  + \sum_{s'\in\SSS} p_\tau (s' | s_\tau,a_\tau) \gamma^*_t(s') - \gamma^*_t(s_\tau) - \sum^{\tau - 1}_{s = t} B_{r, s} - 2D_{\text{max}} \sum^{\tau - 1}_{s = t} B_{p, s} \label{eq:benchmark_by_p},
\end{align} 
where steps (\ref{eq:benchmark_by_r}, \ref{eq:benchmark_by_p}) are by inequalities (\ref{eq:benchmark_for_r}, \ref{eq:benchmark_for_p}). Altogether, the Proposition is proved. $\blacksquare$

\subsection{Extended Value Iteration (EVI) by \citep{JakschOA10}}\label{app:evi}
\begin{algorithm}[!ht]
	\caption{{\sc EVI}$(H^r , H^p; \epsilon)$, mostly extracted from \citep{JakschOA10}}\label{alg:evi}
	\begin{algorithmic}[1]
		\State Initialize VI record $u_0\in \mathbb{R}^{\SSS}$ as $u_0(s) = 0$ for all $s\in \SSS$.
		\For{$i = 0, 1, \ldots$}
		\State For each $s\in \SSS$, compute VI record $u_{i+1}(s) = \max_{a\in \AAA_s}\tilde{\Upsilon}_i(s, a)$, where $$ \tilde{\Upsilon}_i(s, a) = \max_{\dot{r}(s, a)\in H^r(s, a)} \{\dot{r}(s, a)\} + \max_{\dot{p}\in H^p(s, a)}\left\{\sum_{s'\in \SSS} u_i(s')\dot{p}(s')\right\}.$$
		\State Define stationary policy $\tilde{\pi}:\SSS \rightarrow \AAA_s$ as $\tilde{\pi}(s) = \text{argmax}_{a\in \AAA_s}\tilde{\Upsilon}_i(s, a).$
		\State Define optimistic reward $\tilde{r} = \{\tilde{r}(s, a)\}_{s, a}$ with $\tilde{r}(s, a) \in \underset{\dot{r}(s, a)\in H^r(s, a)}{\text{argmax}} \{\dot{r}(s, a)\}$.
		\State Define optimistic distribution $\tilde{p} = \{\tilde{p}(\cdot | s, a)\}_{s, a}$ with $\tilde{p}(\cdot | s, a) \in \underset{\dot{p}\in H^p(s, a)}{\text{argmax}}\left\{\sum_{s'\in \SSS} u_i(s')\dot{p}(s')\right\}$.	
		\State Define optimistic dual variables $\tilde{\rho} = \max_{s\in \SSS}\left\{ u_{i+1}(s) - u_i(s) \right\}$, $\tilde{\gamma}(s) = u_i(s) - \min_{s\in \SSS}u_i(s)$.
		\If{$\max_{s\in \SSS}\left\{ u_{i+1}(s) - u_i(s) \right\} - \min_{s\in \SSS}\left\{ u_{i+1}(s) - u_i(s) \right\} \leq \epsilon$}\label{alg:evi_termination}
		\State Break the \textbf{for} loop.
		\EndIf
		\EndFor
		\State Return policy $\tilde{\pi}$.
		\State Auxiliary output: optimistic reward and state transition distributions $(\tilde{r}, \tilde{p})$, optimistic dual variables $(\tilde{\rho}, \tilde{\gamma})$.	
	\end{algorithmic}
\end{algorithm}
We provide the pseudo-codes of EVI$(H_r, H_p; \epsilon)$ proposed by \citep{JakschOA10} in Algorithm \ref{alg:evi}. By \citep{JakschOA10}, the algorithm converges in finite time when the confidence region $H_p$ contains a transition distribution $p$ such that $(\SSS, \AAA, p)$ constitutes a communicating MDP. The output $(\tilde{\pi}, \tilde{r}, \tilde{p}, \tilde{\rho}, \tilde{\gamma})$ of the EVI$(H_r, H_p; \epsilon)$ satisfies the following two properties \citep{JakschOA10}.
\begin{property}
	\label{property:1}
	The dual variables $(\tilde{\rho}, \tilde{\gamma})$ are optimistic, \ie, $$\tilde{\rho} + \tilde{\gamma}(s) \geq \max_{\dot{r}(s, a) \in H_r(s, a)} \{\dot{r}(s, a)\} + \sum_{s'\in \SSS} \tilde{\gamma}(s') \max_{\dot{p} \in H_p(s, a)} \{\dot{p}(s' | s, a) \}.$$
\end{property}
\begin{property}
	\label{property:2}
	For each state $s\in \SSS$, we have
	$$\tilde{r}(s, \tilde{\pi}(s)) \geq \tilde{\rho} + \tilde{\gamma}(s) - \sum_{s' \in \SSS}\tilde{p}(s' | s, \tilde{\pi}(s)) \tilde{\gamma}(s') - \epsilon.$$	
\end{property}

\textbf{Property 1} ensures the feasibility of the output dual variables $(\tilde{\rho}, \tilde{\gamma})$, with respect to the dual program $\textsf{D}(\dot{r}, \dot{p})$ for any $\dot{r}, \dot{p}$ in the confidence regions $H_{r}, H_p$. The feasibility facilitates the bounding of $\max_{s\in \SSS}\tilde{\gamma}(s)$, which turns out to be useful for bounding the regret arise from switching among different stationary policies. To illustrate, suppose that $H_p$ is so large that it contains a transition distribution $\dot{p}$ under which $(\SSS, \AAA, \dot{p})$ has diameter $D$. By Lemma \ref{lemma:mc}, we have $0\leq \max_{s\in\SSS}\tilde{\gamma}(s) \leq 2D$. 

\textbf{Property 2} ensures the near-optimality of the dual variables $(\tilde{\rho}, \tilde{\gamma})$ to the $(\tilde{r}, \tilde{p})$ optimistically chosen from $H_{r}, H_p$. More precisely, the deterministic policy $\tilde{\pi}$ near-optimal for the MDP with time homogeneous reward function $\tilde{r}$ and time homogeneous transition distribution $\tilde{p}$, under which the policy $\tilde{\pi}$ achieves a long term average reward is at least $\tilde{\rho}^* - \epsilon$. 

\section{Proof of Lemma \ref{lemma:estimation}}
\label{sec:lemma:estimation}
We employ the self-normalizing concentration inequallity \citep{AYPS11}. The following inequality is extracted from Theorem 1 in \citep{AYPS11}, restricted to the case when $d=1$.
\begin{proposition}[\citep{AYPS11}]\label{prop:self-norm}
	Let $\{{\cal F}_q\}^T_{q=1}$ be a filtration. Let $\{\xi_q\}^T_{q=1}$ be a real-valued stochastic process, such that $\xi_q$ is ${\cal F}_q$-measurable, and $\xi_q$ is conditionally $R$-sub-Gaussian, i.e. for all $\lambda \geq 0 $, it holds that $\mathbb{E}[\exp(\lambda \xi_q) | {\cal F}_{q-1}] \leq \exp(\lambda^2 R^2 / 2)$. Let $\{Y_q\}^T_{q=1}$ be a non-negative real-valued stochastic process such that $Y_q$ is ${\cal F}_{q-1}$-measurable. For any $\delta' \in (0, 1)$, it holds that
	\begin{equation*}
		\Pr\left(\frac{\sum^t_{q=1} \xi_q Y_q  }{\max\{1 ,\sum^t_{q=1} Y_q^2 \}} \leq 2 R \sqrt{\frac{\log (T/\delta')}{\max \{1,  \sum^t_{q=1}Y_q^2 \}}} \; \text{ for all $t\in [T]$}\right) \geq 1-\delta'.
	\end{equation*}
	In particular, if $\{Y_q\}^T_{q=1}$ be a $\{0, 1\}$-valued stochastic process, then for any $\delta' \in (0, 1)$, it holds that
	\begin{equation}\label{eq:norm_special}
		\Pr\left(\frac{\sum^t_{q=1} \xi_q Y_q  }{\max\{1 ,\sum^t_{q=1} Y_q \}} \leq 2 R \sqrt{\frac{\log (T/\delta')}{\max \{1,  \sum^t_{q=1}Y_q \}}} \; \text{ for all $t\in [T]$}\right) \geq 1-\delta'.
	\end{equation}
\end{proposition}
The Lemma is proved by applying Proposition \ref{prop:self-norm} with suiatable choices of ${\cal F}^T_{q=1}, \{\xi_q\}^T_{q=1}, \{Y_q\}^T_{q=1}, \delta$. We divide the proof into two parts.
\subsection{Proving $\Pr[{\cal E}_r] \geq 1 - \delta / 2$}
It suffices to prove that, for any fixed $s\in \SSS, a\in \AAA_s, t\in [T]$, it holds that
\begin{align}
&\Pr\left(\left| \hat{r}_(s, a) - \bar{r}_t(s, a)  \right|\leq \rad_{r,t}(s, a)\right) \nonumber\\
= & \Pr\left( \left|\frac{1}{N^+_t(s, a)} \sum^{t-1}_{q = (\tau(m) - W)\vee 1} \left[ R_q(s, a) - r_q(s, a)\right]\cdot \mathbf{1}(s_q = s, a_q = a) \right| \leq 2\sqrt{\frac{\log(2 SAT^2 / \delta)}{N^+_t(s, a)}} \right)\nonumber\\
\geq & 1- \frac{\delta}{2SAT}.\label{eq:to_prove_E_r_each}
\end{align}
since then $\Pr[{\cal E}_r]\geq 1 - \delta / 2$ follows from the union bound over all $s\in \SSS, a\in \AAA_s, t\in [T]$. Now, the trajectory of the online algorithm is expressed as $\{s_q, a_q, R_q\}^T_{q=1}$. Inequality (\ref{eq:to_prove_E_r_each}) directly follows from Proposition \ref{prop:self-norm}, with $\{{\cal F}_q\}^T_{q=1}, \{\xi_q\}^T_{q=1}, \{Y_q\}^T_{q=1}, \delta $ defined as
\begin{align*}
{\cal F}_q &= \{(s_\ell, a_\ell, R_\ell)\}^q_{\ell=1} \cup \{(s_{q+1}, a_{q+1})\},  \\
\xi_q &= R_q(s, a) - r_q(s, a),\\
Y_q &= \mathbf{1}\left(s_q=  s, a_q = a, ((t - W)\vee 1 )\leq q\leq t-1\right), \\
\delta' &= \frac{\delta}{2SAT}.
\end{align*}
Each $\xi_q$ is conditionally 2-sub-Gaussian, since $-1\leq \xi_q\leq 1$ with certainty. Altogether, the required inequality is shown.

\subsection{Proving $\Pr[{\cal E}_p ] \geq 1 - \delta / 2$}
We start by noting that, for two probability distributions $p, \{p(s)\}_{s\in \SSS}, p' = \{p'(s)\}_{s\in \SSS}$ on $\SSS$, it holds that
$$
\left\| p - p'\right\|_1 = \max_{\theta\in \{-1, 1\}^\SSS} \theta(s)\cdot (p(s) - p'(s)).
$$
Consequently, to show $\Pr[{\cal E}_p]\geq 1 - \delta / 2$, it suffices to show that, for any fixed $s\in \SSS, a\in \AAA_s, t\in [T], \theta\in \{-1, 1\}^\SSS$, it holds that
\begin{align}
&\Pr\left( \sum_{s'\in \SSS} \theta(s) \cdot \left( \hat{p}_t(s' | s, a) - \bar{p}_t(s' | s, a)  \right) \leq \rad_{p,t}(s, a)\right) \nonumber\\
\leq & \Pr\left( \frac{1}{N^+_t(s, a)} \sum^{t-1}_{q = (\tau(m) - W)\vee 1} \left[\sum_{s'\in \SSS} \theta(s') \mathbf{1}(s_q = s, a_q = a, s_{q+1} = s' )\right] \right. \nonumber\\
&\qquad \left.  - \left[ \sum_{s'\in \SSS} \theta(s') p_q(s' | s, a) \cdot \mathbf{1}(s_q = s, a_q = a) \right]
\leq 2\sqrt{\frac{\log(2SAT^2 2^S / \delta)}{N^+_t(s, a)}} \right)\nonumber\\
\geq & 1- \frac{\delta}{2SAT 2^S},\label{eq:to_prove_E_p_each}
\end{align}
since then the required inequality follows from a union bound over all $s\in \SSS, a\in \AAA_s, t\in [T], \theta\in \{-1 .1\}^\SSS$. Similar to the casea of  ${\cal E}_r$, (\ref{eq:to_prove_E_p_each}) follows from Proposition \ref{prop:self-norm}, with $\{{\cal F}_q\}^T_{q=1}, \{\xi_q\}^T_{q=1}, \{Y_q\}^T_{q=1}, \delta $ defined as
\begin{align*}
{\cal F}_q &= \{(s_\ell, a_\ell)\}^{q+1}_{\ell=1},  \\
\xi_q &= \left[\sum_{s'\in \SSS} \theta(s') \mathbf{1}(s_q = s, a_q = a, s_{q+1} = s' )\right] - \left[ \sum_{s'\in \SSS} \theta(s') p_q(s' | s, a) \right],\\
Y_q &= \mathbf{1}\left(s_q=  s, a_q = a, ((t - W)\vee 1 )\leq q\leq t-1\right), \\
\delta' &= \frac{\delta}{2SAT 2^S}.
\end{align*}
Each $\xi_q$ is conditionally 2-sub-Gaussian, since $-1\leq \xi_q\leq 1$ with certainty. Altogether, the required inequality is shown.

\section{Proof of Proposition \ref{prop:error}}\label{app:aux_prove_prop_error}
In this section, we prove Proposition \ref{prop:error}. Throughout the section, we impose the assumptions stated by the Proposition. That is, the events ${\cal E}_r, {\cal E}_p$ hold, and there exists $p$ with (1) $p \in H_{p, \tau(m)}(\eta)$, (2) $(\SSS, \AAA, p)$ has diameter at most $D$. We begin by recalling the following notations:
\begin{align*}
B_{r, t} =\max_{s\in \SSS, a\in \AAA_s}\left| r_{t + 1}(s, a) - r_t(s, a)  \right| &,\quad B_{p, t}=\max_{s\in \SSS, a\in \AAA_s }\left\| p_{t + 1}(\cdot | s, a) - p_t ( \cdot  |s, a)  \right\|_1, \nonumber\\
\dev_{r, t} = \sum^{t-1}_{q = \tau(m) - W}B_{r, q} &,\quad \dev_{p, t} =  \sum^{t-1}_{q = \tau(m) - W} B_{p, q}.
\end{align*}
We then need the following auxiliary lemmas
\begin{lemma}\label{lemma:dev_r_p}
	Let $t$ be in episode $m$. For every state-action pair $(s, a)$, we have $$\left| r_t(s, a) - \bar{r}_{\tau(m)} (s, a)\right| \leq \dev_{r, t}, \quad \left\| p_t(\cdot | s, a) - \bar{p}_{ \tau(m)} (\cdot | s, a)\right\|_1 \leq \dev_{p, t}$$
\end{lemma}
\begin{lemma}\label{lemma:dev_rho}
	Let $t$ be in episode $m$. We have 	
	$$\tilde{\rho}_{\tau(m)} \geq \rho^*_t - \dev_{r, t} - 2D\cdot \dev_{p, t}.$$ 
\end{lemma}
\begin{lemma}\label{lemma:cross_p_error}
	Let $t$ be in episode $m$. For every state-action pair $(s, a)$, we have
	\begin{equation*}
	\left|\sum_{s' \in \SSS} \tilde{p}_{\tau(m)}(s' | s, a) \tilde{\gamma}_{\tau(m)}(s')-\sum_{s' \in \SSS} p_t(s' | s, a) \tilde{\gamma}_{\tau(m)}(s')\right| \leq 2 D\left[\dev_{p, t} + 2\rad_{p,\tau (m)}(s, a) + \eta\right].
	\end{equation*}
\end{lemma}
Lemmas \ref{lemma:dev_r_p}, \ref{lemma:dev_rho}, \ref{lemma:cross_p_error} are proved in  Sections \ref{app:pf_claim_dev_r_p},  \ref{app:pf_lemma_dev_rho}, and \ref{app:pf_claim_cross_p_error}, respectively. 
\subsection{Proof of Lemma \ref{lemma:dev_r_p}}\label{app:pf_claim_dev_r_p}
We first provide the bound for rewards:
\begin{align}
\left| r_t(s, a) - \bar{r}_{\tau(m)} (s, a)\right| &\leq \left| r_t(s, a) - r_{\tau (m) } (s, a)\right| + \left| r_{\tau(m)} (s, a) - \bar{r}_{\tau(m)} (s, a)\right| \nonumber\\
& \leq \sum^{t-1}_{q = \tau(m)} \left| r_{q+1}(s, a) - r_q (s, a)\right| + \frac{1}{W}\sum^W_{w = 1}  \left| r_{\tau(m)}(s, a) - r_{\tau(m) - w} (s, a) \right|\nonumber.
\end{align}
By the definition of $B_{r, q}$, we have
\begin{equation*}
\sum^{t-1}_{q = \tau(m)} \left| r_{q+1}(s, a) - r_q (s, a)\right| \leq \sum^{t-1}_{q = \tau(m)} B_{r, q},
\end{equation*}
and 
\begin{align*}
\frac{1}{W}\sum^W_{w = 1} \left| r_{\tau(m)}(s, a) - r_{\tau(m) - w}(s, a) \right| & \leq \frac{1}{W}\sum^W_{w = 1} \sum^w_{i=1} \left| r_{\tau(m)- i + 1}(s, a) - r_{\tau(m) - i}(s, a)\right| \nonumber\\
& \leq \frac{1}{W}\sum^W_{w = 1} \sum^W_{i=1}\left| r_{\tau(m)- i + 1}(s, a) - r_{\tau(m) - i}(s, a)\right| \nonumber\\
& = \sum^W_{i = 1} \left| r_{\tau(m)- i + 1}(s, a) - r_{\tau(m) - i}(s, a)\right| \leq  \sum^W_{i = 1}  B_{r, \tau(m) - i} \nonumber.
\end{align*}
Next, we provide a similar analysis on the transition distribution.
\begin{align}
\left\| p_t(s, a) - \bar{p}_{\tau(m)} (s, a)\right\|_1 &\leq \left\| p_t(s, a) - p_{\tau (m) } (s, a)\right\|_1 + \left\| p_{\tau(m)} (s, a) - \bar{p}_{\tau(m)} (s, a)\right\|_1 \nonumber\\
& \leq \sum^{t-1}_{q = \tau(m)} \left\| p_{q+1}(s, a) - p_q (s, a)\right\|_1 + \frac{1}{W}\sum^W_{w = 1}  \left\| p_{\tau(m)}(s, a) - p_{\tau(m) - w} (s, a) \right\|_1 \nonumber.
\end{align}
By the definition of $B_{p, q}$, we have
\begin{equation*}
\sum^{t-1}_{q = \tau(m)} \left\| p_{q+1}(s, a) - p_q (s, a)\right\|_1 \leq \sum^{t-1}_{q = \tau(m)} B_{p, q},
\end{equation*}
and 
\begin{align*}
\frac{1}{W}\sum^W_{w = 1} \left\| p_{\tau(m)}(s, a) - p_{\tau(m) - w}(s, a) \right\|_1 & \leq \frac{1}{W}\sum^W_{w = 1} \sum^w_{i=1} \left\| p_{\tau(m)- i + 1}(s, a) - p_{\tau(m) - i}(s, a)\right\|_1 \nonumber\\
& \leq \frac{1}{W}\sum^W_{w = 1} \sum^W_{i=1}\left\| p_{\tau(m)- i + 1}(s, a) - p_{\tau(m) - i}(s, a)\right\|_1 \nonumber\\
& = \sum^W_{i = 1} \left\| p_{\tau(m)- i + 1}(s, a) - p_{\tau(m) - i}(s, a)\right\|_1 \leq  \sum^W_{i = 1}  B_{p, \tau(m) - i} \nonumber.
\end{align*}
Altogether, the lemma is shown.	



\subsection{Proof of Lemma \ref{lemma:dev_rho}}\label{app:pf_lemma_dev_rho}
We first demonstrate two immediate consequences about the dual solution $(\tilde{\rho}_{\tau(m)}, \tilde{\gamma}_{\tau(m)})$ by the Proposition's assumptions:
\begin{align}
& 0  \leq \tilde{\gamma}_{\tau(m)}(s)\leq 2 D &\text{ for all $s\in \SSS$},\label{eq:lemma_dev_rho_0}\\
&\tilde{\rho}_{\tau(m)} + \tilde{\gamma}_{\tau(m)}(s) \geq \bar{r}_{\tau(m)}(s, a) + \sum_{s'\in \SSS} \tilde{\gamma}_{\tau(m)}(s') \bar{p}_{\tau(m)}(s' | s, a)&\text{ for all $s\in \SSS, a\in \AAA_s$}.\label{eq:lemma_dev_rho_1}
\end{align}

To see inequality (\ref{eq:lemma_dev_rho_0}), first observe that 
\begin{align}
\tilde{\rho}_{\tau(m)} + \tilde{\gamma}_{\tau(m)}(s) & \geq \max_{\dot{r}(s, a) \in H_{r, \tau(m)}(s, a)} \{\dot{r}(s, a)\} + \sum_{s'\in \SSS} \tilde{\gamma}_{\tau(m)}(s') \max_{\dot{p} \in H_{p, \tau(m)}(s, a;\eta)} \{\dot{p}(s' | s, a) \}\label{eq:lemma_dev_rho_00}\\
& \geq \max_{\dot{r}(s, a) \in H_{r, \tau(m)}(s, a)} \{\dot{r}(s, a)\} + \sum_{s'\in \SSS} \tilde{\gamma}_{\tau(m)}(s') p(s' | s, a)  \label{eq:lemma_dev_rho_01}.
\end{align}
Step (\ref{eq:lemma_dev_rho_00}) is by Property 1 of the output from EVI, which is applied with confidence regions $H_{r, \tau(m)}, H_{p, \tau(m)}(\eta)$. Step (\ref{eq:lemma_dev_rho_01}) is because of the assumption that $p\in H_{p, \tau(m)}(\eta)$. Altogether, the solution $(\tilde{\rho}_{\tau(m)}, \tilde{\gamma}_{\tau(m)})$ is feasible to $\textsf{D}(\dot{r}, p)$ for any $\dot{r}\in H_{r, \tau(m)}$. Now, by Lemma \ref{lemma:mc}, we have $\max_{s, s'\in \SSS} | \tilde{\gamma}_{\tau(m)}(s) - \tilde{\gamma}_{\tau(m)}(s') | \leq 2 D$. Finally, inequality (\ref{eq:lemma_dev_rho_0}) follows from the fact that the bias vector $\tilde{\gamma}_{\tau(m)}$ returned by EVI is component-wise non-negative, and there exists $s\in \SSS$ such that $\tilde{\gamma}_{\tau(m)} = 0$.

To see inequality (\ref{eq:lemma_dev_rho_1}), observe that 
\begin{align}
\tilde{\rho}_{\tau(m)} + \tilde{\gamma}_{\tau(m)}(s) & \geq \max_{\dot{r}(s, a) \in H_{r, \tau(m)}(s, a)} \{\dot{r}(s, a)\} + \sum_{s'\in \SSS} \tilde{\gamma}_{\tau(m)}(s') \max_{\dot{p} \in H_{p, \tau(m)}(s, a;\eta)} \{\dot{p}(s' | s, a) \}\label{eq:lemma_dev_rho_10}\\
& \geq \bar{r}_{\tau(m)}(s, a) + \sum_{s'\in \SSS} \tilde{\gamma}_{\tau(m)}(s') \bar{p}_{\tau(m)}(s' | s, a)  \label{eq:lemma_dev_rho_11}.
\end{align}
Step (\ref{eq:lemma_dev_rho_10}) is again by Property 1 of the output from EVI, and step (\ref{eq:lemma_dev_rho_11}) is by the assumptions that $\bar{r}_{\tau(m)}\in H_{r, \tau(m)}$, and $\bar{p}_{\tau(m)}\in H_{p, \tau(m)} (0)\subset H_{p, \tau(m)} (\eta)$.

Now, we claim that $(\tilde{\rho}_{\tau(m)} + \dev_{r, t} + 2D \cdot \dev_{p, t}, \tilde{\gamma}_{\tau(m)})$ is a feasible solution to the $t$th period dual problem $\mathsf{D}(r_t, p_t)$, which immediately implies the Lemma. To demonstrate the claim, for every state-action pair $(s, a)$ we have
\begin{align}
\bar{r}_{\tau(m)}(s, a) & \geq r_t(s, a) - \dev_{r, t}\label{eq:pf_dev_rho_by_claim_dev_r}\\
\sum_{s'\in \SSS} \tilde{\gamma}_{\tau(m)}(s') p_{\tau(m)}(s' | s, a) & \geq \sum_{s'\in \SSS} \tilde{\gamma}_{\tau(m)}(s') p_t(s' | s, a) - \left\|\tilde{\gamma}_{\tau ( m )}\right\|_\infty \left\| p_t(\cdot | s, a) - \bar{p}_{ \tau(m)} (\cdot | s, a)\right\|_1  \nonumber\\ 
& \geq \sum_{s'\in \SSS} \tilde{\gamma}_{\tau(m)}(s') p_t(s' | s, a) - 2 D \cdot \dev_{p, t}\label{eq:pf_dev_rho_by_lemma_bounded_gamma},.\end{align}
Inequality (\ref{eq:pf_dev_rho_by_claim_dev_r}) is by Lemma \ref{lemma:dev_r_p} on the rewards. Step (\ref{eq:pf_dev_rho_by_lemma_bounded_gamma}) is by inequality (\ref{eq:lemma_dev_rho_0}), and by Lemma \ref{lemma:dev_r_p} which shows $\| p_t(\cdot | s, a) - \bar{p}_{ \tau(m)} (\cdot | s, a)\|_1 \leq \dev_{p, t}$. Altogether, putting (\ref{eq:pf_dev_rho_by_claim_dev_r}), (\ref{eq:pf_dev_rho_by_lemma_bounded_gamma})  to inequality (\ref{eq:lemma_dev_rho_1}), our claim is shown, \ie, for all $s\in\SSS$ and $a\in\AAA_s,$
\begin{align*}
\tilde{\rho}_{\tau(m)} + \dev_{r, t} + 2D \cdot \dev_{p, t}+\tilde{\gamma}_{\tau(m)}(s) \geq r_t(s, a)+\sum_{s'\in \SSS} \tilde{\gamma}_{\tau(m)}(s') p_t(s' | s, a).
\end{align*}
Hence, the lemma is proved. 
\subsection{Proof of Lemma \ref{lemma:cross_p_error}}\label{app:pf_claim_cross_p_error}
We have
\begin{align}
&\left|\sum_{s' \in \SSS} \left[\tilde{p}_{\tau(m)}(s' | s, a) - p_t(s' | s, a)\right] \tilde{\gamma}_{\tau(m)}(s') \right|\nonumber\\
\leq & \underbrace{\left\|\tilde{\gamma}_{\tau(m)}\right\|_\infty}_{(a)} \cdot \left[ \underbrace{\left\| \tilde{p}_{\tau(m)}(\cdot | s, a) - \bar{p}_{\tau(m)}(\cdot | s, a) \right\|_1}_{(b)} + \underbrace{\left\|\bar{p}_{\tau(m)}(\cdot | s, a) - p_t(\cdot | s, a)\right\|_1}_{ ( c )}  \right]  \label{eq:cross_p_by_var_bounds}.
\end{align}
In step (\ref{eq:cross_p_by_var_bounds}), we know that 
\begin{itemize}
	\item $(a) \leq 2 D$ by inequality (\ref{eq:lemma_dev_rho_0}), 
	\item $(b) \leq 2\rad_{p,\tau (m)}(s, a) + \eta$, by the facts that $\tilde{p}_{\tau(m)}(\cdot | s, a)\in H_{p, \tau (m)}(s, a; \eta)$ and  $\bar{p}_{\tau(m)}(\cdot | s, a)\in H_{p, \tau (m)}(s, a; 0)$, 
	\item  $(c)\leq \dev_{p, t}$ by Lemma \ref{lemma:dev_r_p} on the bound on $p$. 
\end{itemize}
Altogether, the Lemma is proved.
\subsection{Finalizing the Proof}
Now, we have
\begin{align}
& r_t(s_t, a_t)\nonumber\\
\geq & \bar{r}_{\tau(m)}(s_t, a_t) - \dev_{r, t}  \label{eq:by_claim_dev_r}\\
\geq & \tilde{r}_{\tau(m)}(s_t, a_t) - \dev_{r, t}  - 2\cdot \rad_{r,\tau(m)}(s_t, a_t)\label{eq:by_event_E_r}\\
\geq & \tilde{\rho}_{\tau(m)} + \tilde{\gamma}_{\tau(m)}(s_t) - \left[ \sum_{s' \in \SSS} \tilde{p}_{\tau(m)}(s' | s_t, a_t) \tilde{\gamma}_{\tau(m)}(s')\right] - \frac{1}{\sqrt{\tau (m)}} \nonumber\\
&\qquad \qquad \qquad - \dev_{r, t}  - 2\cdot \rad_{r,\tau(m)}(s_t, a_t)  \label{eq:by_property_2}\\
\geq & \rho^*_t + \tilde{\gamma}_{\tau(m)}(s_t) - \left[ \sum_{s' \in \SSS} p_t(s' | s_t, a_t) \tilde{\gamma}_{\tau(m)}(s') \right] - \frac{1}{\sqrt{\tau (m)}} \nonumber\\
& - 2 \left[ \dev_{r, t} + \rad_{r,\tau(m)}(s_t, a_t)\right]  - 2 D\left[2 \cdot \dev_{p, t} + 2\cdot \rad_{p,\tau (m)}(s, a) + \eta\right]  \label{eq:by_dev_rho_cross}.
\end{align}
Step (\ref{eq:by_claim_dev_r}) is by Lemma \ref{lemma:dev_r_p} on $t$. Step (\ref{eq:by_event_E_r}) is by conditioning that event ${\cal E}_r$ holds. Step (\ref{eq:by_property_2}) is by Property 2 for the output of EVI. In step (\ref{eq:by_dev_rho_cross}), we upper bound $\tilde{\rho}_{\tau(m)}$ by Lemma \ref{lemma:dev_rho} and we upper bound $\sum_{s' \in \SSS} \tilde{p}_{\tau(m)}(s' | s_t, a_t) \tilde{\gamma}_{\tau(m)}(s')$ by Lemma \ref{lemma:cross_p_error}. Rearranging gives the Proposition.
\section{Proof of Theorem \ref{thm:main1}}
\label{sec:thm:main1}
To facilitate the exposition, we denote $M(T)$ as the total number of episodes. By abusing the notation we, let $\tau (M(T)+ 1) - 1 = T$. Episode $M(T)$, containing the final round $T$, is interrupted and the algorithm is forced to terminate as the end of time $T$ is reached. We can now rewrite the dynamic regret of the \sw~as the sum of dynamic regret from each episode:
\begin{align}
\label{eq:main1}\text{Dyn-Reg}_T(\texttt{SWUCRL2-CW}) = \sum^T_{t=1}\left(\rho^*_t - r_t(s_t, a_t)\right) = \sum^{M(T)}_{m=1}\sum^{\tau(m+1) - 1}_{t = \tau(m)} \left(\rho^*_t - r_t(s_t, a_t)\right)
\end{align}

To proceed, we define the set $$ U = \{m\in [M(T)] : p_{\tau(m)}(\cdot|s,a)\in H_{p,\tau(m)}(s,a;\eta)~\forall (s,a)\in \SSS\times\AAA_s\}.$$ For each episode $m\in[M(T)]$, we distinguish two cases:
\begin{itemize}
	\item \textbf{Case 1.} $m\in U.$ Under this situation, we apply Proposition \ref{prop:error} to bound the dynamic regret during the episode, using the fact that $p_{\tau(m)}$ satisfies the assumptions of the proposition with $D = D_{\tau(m)}\leq D_\text{max}$.
	\item \textbf{Case 2.} $m\in [M(T)] \setminus U$. In this case, we trivially upper bound the dynamic regret of each round in episode $m$ by $1.$
\end{itemize}    

For case 1, we bound the dynamic regret during episode $m$ by summing the error terms in (\ref{eq:prop_error_1}, \ref{eq:prop_error_2}) across the rounds $t\in [\tau(m) , \tau(m+1) - 1]$ in the episode. The term (\ref{eq:prop_error_1}) accounts for the error by switching policies. In (\ref{eq:prop_error_2}), the terms $\rad_{r, \tau(m)}, \rad_{p, \tau(m)}$ accounts for the estimation errors due to stochastic variations, and the term $\dev_{r, t}, \dev_{p, t}$ accounts for the estimation error due to non-stationarity.

For case 2, we need an upper bound on $\sum_{m \in [M(T)] \setminus U}\sum^{\tau(m+1) - 1}_{t = \tau(m)} 1$, the total number of rounds that belong to an episode in $[M(T)] \setminus U$. 
The analysis is challenging, since the length of each episode may vary, and one can only guarantee that the length is $\leq W$. A first attempt could be to upper bound as $\sum_{m \in [M(T)] \setminus U}\sum^{\tau(m+1) - 1}_{t = \tau(m)} 1 \leq W \sum_{m \in [M(T)] \setminus U} 1$, but the resulting bound appears too loose to provide any meaningful regret bound. Indeed, there could be double counting, as the starting time steps for a pair of episodes in case 2 might not even be $W$ rounds apart!

To avoid the trap of double counting, we consider a set $Q_T\subseteq [M(T)] \setminus U$ where the start times of the episodes are sufficiently far apart, and relate the cardinality of $Q_T$ to $\sum_{m \in [M(T)] \setminus U} \sum^{\tau(m+1) - 1}_{t = \tau(m)} 1$. 
The set $Q_T\subseteq[M(T)]$ is constructed sequentially, by examining all episodes $m = 1 , \ldots, M(T)$ in the time order. At the start, we initialize $Q_T = \emptyset$. For each $m = 1, \ldots, M(T)$, we perform the following. If episode $m$ satisfies both criteria:
\begin{enumerate}
	\item There exists some $s\in\SSS$ and $a\in\AAA_s$ such that $p_{\tau(m)}(\cdot|s,a)\notin H_{p,\tau(m)}(s,a;\eta);$
	\item For every $m'\in Q_T,$ $\tau(m)-\tau(m')>W,$
\end{enumerate}
then we add $m$ into $Q_T.$ Afterwards, we move to the next episode index $m+1$. The process terminates once we arrive at episode $M(T)+1.$ The construction ensures that, for each episode $m\in[M(T)],$ if $\tau(m)-\tau(m')\notin[0,W]$ for all $m'\in Q_T,$ then $\forall s\in\SSS~\forall a\in\AAA_s~p_{\tau(m)}(\cdot|s,a)\in H_{p,\tau(m)}(s,a);$ otherwise, $m$ would have been added into $Q_T.$ 

We further construct a set $\tilde{Q}_T$ to include all elements in $Q_T$ and every episode index $m$ such that there exists $m'\in Q_T$ with $\tau(m)-\tau(m')\in[0,W].$ By doing so, we can prove that every episode $m\in[M(T)]\setminus \tilde{Q}_T$ satisfies $p_{\tau(m)}(\cdot|s,a)\in H_{p,\tau(m)}(s,a)~\forall s\in\SSS~\forall a\in\AAA_s.$ The procedures for building $\tilde{Q}_T$ (initialized to $Q_T$) are described as follows: for every episode index $m\in[M(T)],$ if there exists $m'\in Q_T,$ such that $\tau(m)-\tau(m')\in[0,W],$ then we add $m$ to $\tilde{Q}_T.$ Formally,
\begin{align*}
\tilde{Q}_T=Q_T\cup\left\{m\in [M(T)]:\exists m'\in Q_T~\tau(m)-\tau(m')\in[0,W]\right\}.
\end{align*}
\begin{figure}[t]
	\centering
	\includegraphics[width=13cm,height=1.8cm]{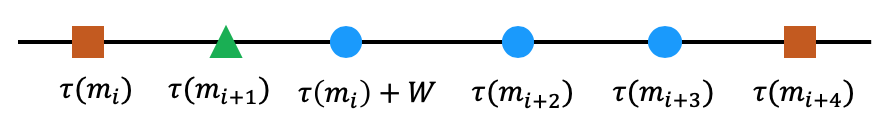}
	\caption{Both episodes $m_i$ and $m_{i+4}$ belong to $Q_T$ (and thus $\tilde{Q}_T$) because $p_{\tau(m_{i})}\notin H_{p,\tau(m_{i})}(\eta)$ and $p_{\tau(m_{i+4})}\notin H_{p,\tau(m_{i+4})}(\eta).$ $m_{i+1}$ is added to $\tilde{Q}_T$ (but not $Q_T$) because $\tau(m_{i+1})-\tau(m_i)\in[0,W].$ $m_{i+2}$ and $m_{i+3}$ belong to neither of $Q_T$ nor $\tilde{Q}_T$ as $p_{\tau(m_{i+2})}\in H_{p,\tau(m_{i+2})}(\eta)$ and $p_{\tau(m_{i+3})}\in H_{p,\tau(m_{i+3})}(\eta).$}
	\label{fig:Q}
\end{figure}
We can formalize the properties of $Q_T$ and $\tilde{Q}_T$ as follows.
\begin{lemma}
	\label{lemma:Q_size}
	Conditioned on $\mathcal{E}_p,$
	$|Q_T|\leq{B_p}/{\eta}.$
\end{lemma}
\begin{lemma}
	\label{lemma:Q_tilde}
	For any episode $m\notin\tilde{Q}_T,$ we have $p_{\tau(m)}(\cdot|s,a)\in H_{p,\tau(m)}(s,a;\eta)$ for all $s\in\SSS$ and $a\in\AAA_s.$
\end{lemma}
The proofs of Lemmas \ref{lemma:Q_size} and \ref{lemma:Q_tilde} are presented in Sections \ref{sec:lemma:Q_size} and \ref{sec:lemma:Q_tilde}, respectively.

Together with eqn. \eqref{eq:main1}, we can further decompose the dynamic regret of the \sw~as
\begin{align}
\nonumber&\text{Dyn-Reg}_T(\texttt{SWUCRL2-CW})\\
\nonumber=&\sum_{m\in\tilde{Q}_T}\sum^{\tau(m+1) - 1}_{t = \tau(m)} \left(\rho^*_t - r_t(s_t, a_t)\right)+\sum_{m\in[M_T]\setminus\tilde{Q}_T}\sum^{\tau(m+1) - 1}_{t = \tau(m)} \left(\rho^*_t - r_t(s_t, a_t)\right)\\
\tag{$\spadesuit$}\label{eq:Q_error}	\leq &\sum_{m\in\tilde{Q}_T}\sum^{\tau(m+1) - 1}_{t = \tau(m)} \left(\rho^*_t - r_t(s_t, a_t)\right)\\
&\quad + \sum_{m\in[M_T]\setminus\tilde{Q}_T}\sum^{\tau(m+1) - 1}_{t = \tau(m)} \left\{ \left[ \sum_{s' \in \SSS} p_t(s' | s_t, a_t) \tilde{\gamma}_{\tau(m)}(s') - \tilde{\gamma}_{\tau(m)}(s_t)\right]   + \frac{1}{\sqrt{\tau (m)}}\right\}\tag{$\clubsuit$}\label{eq:episode_error}\\
& \quad + \sum_{m\in[M_T]\setminus\tilde{Q}_T}\sum^{\tau(m+1) - 1}_{t = \tau(m)}\left( 2  \dev_{r, t} + 4 D_{\text{max}} \cdot \dev_{p, t} + 2 D_{\text{max}}\eta \right)\tag{$\vardiamond$} \label{eq:adv_error}\\
& \quad  + \sum_{m\in[M_T]\setminus\tilde{Q}_T}\sum^{\tau(m+1) - 1}_{t = \tau(m)}\left[2 \rad_{r,\tau(m)}(s_t, a_t) + 4 D_{\text{max}}\cdot \rad_{p,\tau (m)}(s_t, a_t)  \right]\tag{$\varheart$} \label{eq:sto_error},
\end{align}
where the last step makes use of Lemma \ref{lemma:Q_tilde} and Proposition \ref{prop:error}.  We accomplish the promised dynamic regret bound by the following four Lemmas that bound the dynamic regret terms (\ref{eq:Q_error}, \ref{eq:episode_error}, \ref{eq:adv_error}, \ref{eq:sto_error}).
\begin{lemma}
	\label{lemma:Q_error}
	Conditioned on $\mathcal{E}_p,$ we have $$\eqref{eq:Q_error}=O\left(\frac{B_pW}{\eta}\right).$$
\end{lemma} 
\begin{lemma}\label{lemma:episode_errror}
	Conditioned on events ${\cal E}_r, {\cal E}_p$, we have with probability at least $1 - O(\delta)$ that 
	$$\eqref{eq:episode_error} = \tilde{O}( D_\text{max} [M(T) + \sqrt{T}] ) =\tilde{O}\left( D_\text{max} \left[\frac{SAT}{W} + \sqrt{T}\right] \right).$$
\end{lemma}
\begin{lemma}\label{lemma:adv_error} 
	With certainty, $$\eqref{eq:adv_error} = O\left((B_r + D_\text{max} B_p) W + D_{\text{max}} T\eta\right).$$
\end{lemma}
\begin{lemma}\label{lemma:sto_error}
	With certainty, we have 
	$$\eqref{eq:sto_error} = \tilde{O}\left( \frac{D_\text{max} S\sqrt{A} T}{\sqrt{W}}\right) .$$
\end{lemma}
The proofs of Lemmas \ref{lemma:Q_error}, \ref{lemma:episode_errror}, \ref{lemma:adv_error}, and \ref{lemma:sto_error} are presented in Sections \ref{sec:lemma:Q_error}, \ref{app:pf_lemma_episode_error}, \ref{app:pf_lemma_adv_error}, and \ref{app:pf_lemma_sto_error}, respectively. Putting all these pieces together, we have the dynamic regret of the \sw~is upper bounded as
\begin{align*}
\tilde{O}\left(\frac{B_pW}{\eta}+B_r W + D_\text{max}\left[ B_p W + \frac{S\sqrt{A}T}{\sqrt{W}} + T\eta + \frac{SAT}{W}  + \sqrt{T}\right]\right),
\end{align*}
and by setting $W$ and $\eta$ accordingly, we can conclude the proof. 
\subsection{Proof of Lemma \ref{lemma:Q_size}}
\label{sec:lemma:Q_size} 
We first claim that, for every episode $m\in Q_T$, there exists some state-action pair $(s,a)$ and some time step $t_m\in [(\tau(m)-W \vee 1),\tau(m)-1]$ such that 
\begin{align}
\label{eq:Q_size1}
\|p_{\tau(m)}(\cdot|s,a)-p_{t_m}(\cdot|s,a)\|_1 >\eta.
\end{align}
For contradiction sake, suppose the otherwise, that is, $\| p_{\tau(m)}(\cdot | s, a) - p_t(\cdot | s, a) \|_1 \leq \eta$ for every state-action pair $s, a$ and every time step $t\in [(\tau(m)-W \vee 1),\tau(m)-1]$. For each state-action pair $(s, a)$, consider the following cases on $N_{\tau(m)}(s, a) = \sum^{\tau(m) - 1}_{q = (\tau(m) - W)\vee 1} \mathbf{1}(s_q = s, a_q = a)$:
\begin{itemize}
	\item \textbf{Case 1: $N_{\tau(m)}(s, a) = 0$.} Then $\hat{p}_{\tau(m)}(\cdot | s, a) = \mathbf{0}^\SSS$, and clearly we have $$\|p_{\tau(m)}(\cdot | s, a) - \hat{p}_{\tau(m)}(\cdot | s, a) \|_1 = 1 < \rad p_{\tau(m)}(s, a) < \rad p_{\tau(m)}(s,a) + \eta.$$ 
	\item \textbf{Case 2:  $N_{\tau(m)}(s, a) > 0$.} Then we have the following inequalities:
	\begin{align}
	&\|p_{\tau(m)}(\cdot|s,a)-\bar{p}_{\tau(m)} (\cdot | s, a)\|_1 \nonumber\\
	=& \left\|\frac{1}{N^+_{\tau(m)}(s, a)}  \sum^{\tau(m)-1}_{q =  (\tau(m) - W) \vee 1 } \left[p_{\tau(m)}(\cdot|s,a)- p_q(\cdot | s, a)\right] \cdot \mathbf{1}(s_q = s, a_q = a)\right\|_1\label{eq:Q_size_by_bar_p_def}\\
	\leq&\frac{1}{N^+_{\tau(m)}(s, a)} \sum^{\tau(m)-1}_{q =  (\tau(m) - W) \vee 1 } \left\|p_{\tau(m)}(\cdot|s,a)-p_q(\cdot| s, a) \right\|_1 \cdot \mathbf{1}(s_q = s, a_q = a) \leq \eta.\label{eq:Q_size_by_tri_ineq}
	\end{align}
	Step (\ref{eq:Q_size_by_bar_p_def}) is by the definition of $\bar{p}_{\tau(m)}(\cdot | s, a)$, and step (\ref{eq:Q_size_by_tri_ineq}) is by the triangle inequality.
	Consequently, we have
	\begin{align}
	&\|p_{\tau(m)}(\cdot|s,a)-\hat{p}_{\tau(m)} (\cdot | s, a)\|_1\nonumber\\
	\leq&\|\bar{p}_{\tau(m)}(\cdot|s,a)-\hat{p}_{\tau(m)} (\cdot | s, a)\|_1+\|p_{\tau(m)}(\cdot|s,a)-\bar{p}_{\tau(m)} (\cdot | s, a)\|_1\label{eq:Q_size_by_E_p}\\
	\leq &\rad p_{\tau(m)}(s, a)+\eta. \nonumber
	\end{align} 
	Step (\ref{eq:Q_size_by_E_p}) is true since we condition on the event ${\cal E}_p$, 
\end{itemize}
Combining the two cases, we have shown that $p_{\tau(m)}(\cdot | s, a) \in H_{p, \tau(m)}(s, a, \eta)$ for all $s\in \SSS, a\in \AAA_s$, contradicting to the fact that $m\in Q_T$. Altogether, our claim on inequality (\ref{eq:Q_size1}) is established.

Finally, we provide an upper bound to $|Q_T|$ using (\ref{eq:Q_size1}):
\begin{align}
B_p=&\sum_{t\in[T-1]}\max_{s\in \SSS, a\in \AAA_s }\left\{ \left\| p_{t + 1}(\cdot | s, a) - p_t ( \cdot  |s, a)  \right\|_1\right\}\nonumber \\
\geq&\sum_{m\in Q_T}\sum_{q=t_m}^{\tau(m)-1}\max_{s\in \SSS, a\in \AAA_s }\left\{ \left\| p_{q + 1}(\cdot | s, a) - p_q ( \cdot  |s, a)  \right\|_1\right\}\label{eq:Q_size2} \\
\geq & \sum_{m\in Q_T} \max_{s\in \SSS, a\in \AAA_s }\left\{ \left\|  \sum_{q=t_m}^{\tau(m)-1} (p_{q + 1}(\cdot | s, a) - p_q ( \cdot  |s, a) ) \right\|_1\right\}\label{eq:Q_size3} \\
>& |Q_T|  \eta. \label{eq:Q_size4}
\end{align}
Step \eqref{eq:Q_size2} follows by the second criterion of the construction of $Q_T,$ which ensures that for distinct $m, m'\in Q_T$, the time intervals $[t_m,\tau(m)]$, $[t_{m'},\tau(m')]$ are disjoint. Step (\ref{eq:Q_size3}) is by applying the triangle inequality, for each $m\in Q_T$, on the state-action pair $(s, a)$ that maximizes the term $\|  \sum_{q=t_m}^{\tau(m)-1} (p_{q + 1}(\cdot | s, a) - p_q ( \cdot  |s, a) ) \|_1 = \| (p_{\tau(m)}(\cdot | s, a) - p_{t_m} ( \cdot  |s, a) ) \|_1$. Step (\ref{eq:Q_size4}) is by applying the claimed inequality (\ref{eq:Q_size1}) on each $m\in Q_T$. Altogether, the Lemma is proved.  $\blacksquare$


\subsection{Proof of Lemma \ref{lemma:Q_tilde}}
\label{sec:lemma:Q_tilde} 
We prove by contradiction. Suppose there exists an episode $m\notin\tilde{Q}_T,$ a state $s\in\SSS,$ and an action $a\in\AAA_s$ such that
$p_{\tau(m)}(\cdot|s,a)\notin H_{p,\tau(m)}(s,a;\eta),$ then $m$ should have been added to $Q_T.$  To see this, we first note that episode $m$ trivially satisfies criterion 1 in the construction of $Q_T.$ Moreover, at the time when $m$ is examined, we know that any $m'$ has been added to $Q_T$ should satisfy $\tau(m)-\tau(m')>W$ as otherwise $m$ would have been added to $\tilde{Q}_T.$ Therefore, we have prove $m\in Q_T\subseteq\tilde{Q}_T,$ which is clearly a contradiction.
\subsection{Proof of Lemma \ref{lemma:Q_error}}
\label{sec:lemma:Q_error}
Denote $Q_T=\{m_1,\ldots,m_{|Q_T|}\}.$ By construction, for every element $m\in\tilde{Q}_T,$ there exists an unique $m'\in Q_T$ such that 
\begin{align}\tau(m)-\tau(m')\in[0,W].\end{align} 
We can thereby partition the elements of $\tilde{Q}_T$ into $|Q_T|$ disjoint subsets $\tilde{Q}_T(m_1),\ldots,\tilde{Q}_T(m_{|Q_T|})$ such that 
\begin{enumerate}
	\item Each element $m\in\tilde{Q}_T$ belongs to exactly one $\tilde{Q}_T(m')$ for some $m'\in Q_T.$
	\item Each element $m\in\tilde{Q}_T(m')$ satisfies $\tau(m)-\tau(m')\in[0,W].$
\end{enumerate}
We bound \eqref{eq:Q_error} from above as
\begin{align}
\nonumber&\sum_{m'\in{Q}_T}\sum_{m\in \tilde{Q}_T(m')}\sum^{\tau(m+1) - 1}_{t = \tau(m)} \left(\rho^*_t - r_t(s_t, a_t)\right)\\
\label{eq:Q_error1}\leq&\sum_{m'\in{Q}_T}\sum_{m\in \tilde{Q}_T(m')}\sum^{\tau(m+1) - 1}_{t = \tau(m)} 1\\
\nonumber=&\sum_{m'\in{Q}_T}\sum_{m\in \tilde{Q}_T(m')}\left(\tau(m+1)-\tau(m)\right)\\
\label{eq:Q_error2}\leq&\sum_{m'\in{Q}_T}\left(\max_{m\in \tilde{Q}_T(m')} \tau(m+1)-\tau(m')\right)\\
\nonumber=&\sum_{m'\in{Q}_T}\left[\max_{m\in \tilde{Q}_T(m')}\left( \tau(m+1)-\tau(m)+\tau(m)\right)-\tau(m')\right]\\
\nonumber\leq&\sum_{m'\in{Q}_T}\left[\max_{m\in \tilde{Q}_T(m')}\left( \tau(m+1)-\tau(m)\right)+\max_{m\in \tilde{Q}_T(m')}\tau(m)-\tau(m')\right]\\
\label{eq:Q_error3}\leq&\sum_{m'\in{Q}_T}\left[W+W\right]\\
\nonumber=&2W|Q_T|\\
\nonumber\leq & 2B_pW/\eta.
\end{align}
Here, inequality \eqref{eq:Q_error1} holds by boundedness of rewards, inequality \eqref{eq:Q_error2} follows from the fact that episodes are mutually disjoint, inequality \eqref{eq:Q_error3} makes the observations that each episode can last for at most $W$ time steps (imposed by the \sw) as well as criterion 2 of the construction of $\tilde{Q}_T(m')$'s, and the last step uses Lemma \ref{lemma:Q_size}.
\subsection{Proof of Lemma \ref{lemma:episode_errror}}\label{app:pf_lemma_episode_error}
We first give an upper bound for $M(T),$ the total number of the episodes.
\begin{lemma}\label{lemma:number_episode}
	Conditioned on events ${\cal E}_r, {\cal E}_p$, we have $M(T) \leq SA( 2+ \log_2 W) T / W = \tilde{O}\left(SAT / W\right)$ with certainty.
\end{lemma}
\begin{proof}
	First, to demonstrate the bound for $M(T)$, it suffices to show that there are at most $SA( 2 + \log_2 W)$ many episodes in each of the following cases: (1) between time steps $1$ and $W$, (2) between time steps $j W$ and $(j + 1)W$, for any $j\in \{1, \ldots, \lfloor T/ W\rfloor - 1 \}$, (3) between time steps $\lfloor T/ W\rfloor \cdot W$ and $T$. We focus on case (2), and the edge cases (1, 3) can be analyzed similarly.
	
	Between time steps $j W$ and $(j + 1)W$, a new episode $m+1$ is started when the second criterion $\nu_m(s_t, \tilde{\pi}_m(s_t)) < N^+_{\tau(m)}(s_t, \tilde{\pi}_m(s_t))$ is violated during the current episode $m$. We first illustrate how the second criterion is invoked for a fixed state-action pair $(s, a)$, and then bound the number of invocations due to $(s, a)$. Now, let's denote $m^1, \ldots, m^L$ as the episode indexes, where $j W \leq \tau(m^1) < \tau(m^2) < \ldots < \tau (m^L) < (j + 1)W$, and the second criterion for $(s, a)$ is invoked during $m^\ell$ for $1\leq \ell \leq L$. That is, for each $\ell\in \{1, \ldots, L\}$, the DM breaks the \textbf{while} loop for episode $m^\ell$ because $\nu_{m^\ell}(s, a) = N^+_{\tau(m^\ell)}(s, a)$, leading to the new episode $m^\ell +1$. 
	
	To demonstrate our claimed bound for $M(T)$, we show that $L \leq 2 + \log_2 W$ as follows. To ease the notation, let's denote $\psi^\ell = \nu_{m^\ell}(s, a)$. We first observe that $\psi^1 = N^+_{\tau(m^1)}(s, a) \geq 1$. Next, we note that for $\ell \in \{2, \ldots L\}$, we have $\psi^\ell \geq \sum^{\ell - 1}_{k  = 1}\psi^k$. \footnote{We proceed slightly differently from the stationary case, where the corresponding $N_t(s, a)$ is non-decreasing in $t$ \citep{JakschOA10}, which is clearly not true here due to the use of sliding windows} Indeed, we know that for each $\ell$ we have $(\tau( m ^\ell + 1) - 1) - \tau(m^1) \leq W$, by our assumption on $m^1, \ldots, m^\ell$. Consequently, the counting sum in $N_{\tau(m^\ell)}(s, a)$, which counts the occurrences of $(s, a)$ in the previous $W$ time steps, must have counted those occurrences corresponding to $\psi^1, \ldots, \psi^{\ell - 1}$. The worst case sequence of $\psi^1, \psi^2, \ldots, \psi^L$ that yields the largest $L$ is when $\psi^1 = \psi^2 = 1$, $\psi^3 = 2$, and more generally $\psi^\ell = 2^{\ell - 2}$ for $\ell \geq 2$. Since $\psi^\ell \leq W$ for all $W$, we clearly have $L \leq 2 +\log_2 W$, proving our claimed bound on $L$.
	Altogether, during the $T$ time steps, there are at most $(SAT(2 +\log_2 W )) / W$ episodes due to the second criterion and $T/W$ due to the first, leading to our desired bound on $M(T)$.	
\end{proof}

Next, we establish the bound for \eqref{eq:episode_error}.  By Lemma \ref{lemma:Q_tilde}, we know that $\tilde{\gamma}_{\tau (m)}(s) \in [0, 2 D_\text{max}]$ for all $m\in[M(T)]\setminus \tilde{Q}_T$ and $s.$ For each episode $m\in[M(T)]\setminus \tilde{Q}_T$, we have 
\begin{align}
&\sum^{\tau(m+1) - 1}_{t = \tau(m)}  \left[ \sum_{s' \in \SSS} p_t(s' | s_t, a_t) \tilde{\gamma}_{\tau(m)}(s') - \tilde{\gamma}_{\tau(m)}(s_t)\right] \nonumber\\
= & \underbrace{- \tilde{\gamma}_{\tau(m)}(s_{\tau(m)}) + \tilde{\gamma}_{\tau(m)}(s_{\tau(m) + 1})}_{\leq 2 D_\text{max}} +\sum^{\tau(m+1) - 1}_{t = \tau(m)}  \underbrace{\left[ \sum_{s' \in \SSS} p_t(s' | s_t, a_t) \tilde{\gamma}_{\tau(m)}(s') - \tilde{\gamma}_{\tau(m)}(s_{t+1})\right]}_{= Y_t}.\label{eq:pf_episode_error_terms}
\end{align}
Summing (\ref{eq:pf_episode_error_terms}) over $m\in[M(T)]\setminus \tilde{Q}_T$ we get 
\begin{align}
\nonumber&\sum_{m\in[M(T)]\setminus \tilde{Q}_T}\sum^{\tau(m+1) - 1}_{t = \tau(m)}  \left[ \sum_{s' \in \SSS} p_t(s' | s_t, a_t) \tilde{\gamma}_{\tau(m)}(s') - \tilde{\gamma}_{\tau(m)}(s_t)\right]  \\
\leq &2 D_\text{max}\cdot \left(M(T)-|\tilde{Q}_T|\right)+\sum_{m\in[M(T)]\setminus \tilde{Q}_T}\sum^{\tau(m+1) - 1}_{t = \tau(m)}Y_t.
\end{align}
Define the filtration ${\cal H}_{t-1} = \{(s_q, a_q, R_q(s_q, a_q))\}^t_{q=1}$. Now, we know that $\mathbb{E}[Y_t | {\cal H}_{t-1}] = 0$, $Y_t$ is $\sigma({\cal H}_t)$-measurable, and $| Y_t | \leq 2 D_\text{max}$. Therefore, we can apply the Hoeffding inequality \citep{Hoeffding63} to show that 
$$
\sum_{m\in[M(T)]\setminus \tilde{Q}_T}\sum^{\tau(m+1) - 1}_{t = \tau(m)} \left[ \sum_{s' \in \SSS} p_t(s' | s_t, a_t) \tilde{\gamma}_{\tau(m)}(s') - \tilde{\gamma}_{\tau(m)}(s_t)\right]   = O\left(D_\text{max} \cdot  M(T) + D_\text{max} \sqrt{T\log \frac{1}{\delta}}\right)
$$
with probability $1 - O(\delta),$ where we use the facts that $M(T)-|\tilde{Q}_T|\leq M(T)$ and $\sum_{m\in[M(T)]\setminus \tilde{Q}_T}\sum^{\tau(m+1) - 1}_{t = \tau(m)}1\leq T.$ Finally, note that
$$
\sum_{m\in[M(T)]\setminus \tilde{Q}_T}\sum^{\tau(m+1) - 1}_{t = \tau(m)}\frac{1}{\sqrt{\tau (m)}}\leq\sum^{M(T)}_{m=1}\sum^{\tau(m+1) - 1}_{t = \tau(m)} \frac{1}{\sqrt{\tau (m)}} \leq \sum^{\lceil T / W \rceil}_{i = 1} \frac{W}{\sqrt{i W}} = O(\sqrt{T}).
$$
Hence, the Lemma is proved.
\subsection{Proof of Lemma \ref{lemma:adv_error}}\label{app:pf_lemma_adv_error}
We first note that
$$\sum_{m\in[M_T]\setminus\tilde{Q}_T}\sum^{\tau(m+1) - 1}_{t = \tau(m)}\left(2  \dev_{r, t} + 4 D_{\text{max}} \cdot \dev_{p, t} + 2 D_{\text{max}}\eta \right)\leq\sum_{t=1}^T\left(2  \dev_{r, t} + 4 D_{\text{max}} \cdot \dev_{p, t} + 2 D_{\text{max}}\eta \right),$$
since $\dev_{r, t}\geq 0$ and $\dev_{p, t}\geq 0$ for all $t$. 
We can thus work with the latter quantity. 

We first bound $\sum^T_{t=1} \dev_{r, t}$. Now, recall the definition that, for time $t$ in episode $m$, we have defined $\dev_{r, t} = \sum^{t-1}_{q = \tau(m) - W} B_{r, q}$. Clearly, for $i W \leq q <  (i+1) W$, the summand $B_{r, q}$ only appears in $\dev_{r, t}$ for $ i W \leq q < t \leq (i + 2) W $, since each episode is contained in $\{i' W , \ldots, (i' + 1) W\}$ by our episode termination criteria ($t$ is a multiple of $W$) of the \sw. Altogether, we have
\begin{equation}\label{eq:pf_lemma_adv_error_r}
2\sum^T_{t=1} \dev_{r, t} \leq 2\sum^{T-1}_{t=1} B_{r, t}  W  = 2 B_r W . 
\end{equation}
Next, we bound $\sum^T_{t=1} \dev_{p, t}$. Now, we know that $\tau(m+1) - \tau(m)\leq W$ by our episode termination criteria ($t$ is a multiple of $W$) of the \sw. Consequently, 
\begin{equation}\label{eq:pf_lemma_adv_error_p}
4 D_\text{max} \sum^T_{t=1} \dev_{p, t} \leq 4 D_\text{max}\sum^{T-1}_{t=1} B_{p, t}  W = 4 D_\text{max} B_p W . 
\end{equation}
Finally, combining (\ref{eq:pf_lemma_adv_error_r}, \ref{eq:pf_lemma_adv_error_p}) with $2 D_\text{max} \sum^T_{t=1} \eta  $, the Lemma is established.
\subsection{Proof of Lemma \ref{lemma:sto_error}}\label{app:pf_lemma_sto_error}
Due to non-negativity of $\rad r_{t}(s,a)$'s and $\rad p_{t}(s,a)$'s, we have
\begin{align}
&\sum_{m\in[M_T]\setminus\tilde{Q}_T}\sum^{\tau(m+1) - 1}_{t = \tau(m)} \rad r_{\tau(m)}(s_t, a_t)\leq\sum^{M(T)}_{m=1}\sum^{\tau(m+1) - 1}_{t = \tau(m)} \rad r_{\tau(m)}(s_t, a_t),\\
&\sum_{m\in[M_T]\setminus\tilde{Q}_T}\sum^{\tau(m+1) - 1}_{t = \tau(m)} \rad p_{\tau(m)}(s_t, a_t)\leq\sum^{M(T)}_{m=1}\sum^{\tau(m+1) - 1}_{t = \tau(m)} \rad p_{\tau(m)}(s_t, a_t)
\end{align}
We thus first show that, with certainty,
\begin{align}
&\sum^{M(T)}_{m=1}\sum^{\tau(m+1) - 1}_{t = \tau(m)} \rad r_{\tau(m)}(s_t, a_t) = \sum^{M(T)}_{m=1}\sum^{\tau(m+1) - 1}_{t = \tau(m)}2\sqrt{\frac{ 2\ln ( SA T / \delta ) }{N^+_{\tau (m)} (s_t, a_t)}} = \tilde{O}\left(\frac{\sqrt{SA}T}{ \sqrt{W}}\right), \label{eq:pf_sto_error_for_r}\\
&\sum^{M(T)}_{m=1}\sum^{\tau(m+1) - 1}_{t = \tau(m)} \rad p_{\tau(m)}(s_t, a_t) = \sum^{M(T)}_{m=1}\sum^{\tau(m+1) - 1}_{t = \tau(m)}2\sqrt{\frac{2S\log\left(ATW/\delta\right)}{N^+_{\tau (m)} (s_t, a_t)}} = \tilde{O}\left(\frac{S\sqrt{A}T}{ \sqrt{W}}\right). \label{eq:pf_sto_error_for_p}
\end{align}
We analyze by considering the dynamics of the algorithm in each consecutive block of $W$ time steps, in a way similar to the proof of Lemma  \ref{lemma:episode_errror}. Consider the episodes indexes $m_0, m_1 \ldots, m_{\lceil T / W\rceil}, m_{\lceil T / W\rceil + 1},$ where $\tau(m_0) = 1$, and $\tau(m_j) = j W$ for $j\in \{1, \ldots, \lceil T / W\rceil\}$, and $m_{\lceil T / W\rceil + 1} = m(T) + 1$ (so that $\tau(m_{\lceil T / W\rceil + 1} - 1)$ is the time index for the last episode in the horizon).  

To prove (\ref{eq:pf_sto_error_for_r}, \ref{eq:pf_sto_error_for_p}), it suffices to show that, for each $j\in \{0, 1, \ldots, \lfloor T / W \rfloor\}$, we have
\begin{equation}\label{eq:pf_sto_error_suffice}
\sum^{m_{j+1} - 1}_{m = m_j} \sum^{\tau(m + 1) - 1}_{t = \tau(m)} \frac{1}{\sqrt{N^+_{\tau (m)} (s_t, a_t)}} = O\left( \sqrt{SAW} \right).
\end{equation}
Without loss of generality, we assume that $j \in \{1, \ldots, \lfloor T / W \rfloor - 1\}$, and the edge cases of $j = 0, \lfloor T / W \rfloor$ can be analyzed similarly. 

Now, we fix a state-action pair $(s, a)$ and focus on the summands in (\ref{eq:pf_sto_error_suffice}):
\begin{equation}\label{eq:pf_sto_error_1st}
\sum^{m_{j+1} - 1}_{m = m_j} \sum^{\tau(m + 1) - 1}_{t = \tau(m)} \frac{\mathbf{1}((s_t, a_t) = (s, a))}{\sqrt{N^+_{\tau (m)} (s_t, a_t)}} = \sum^{m_{j+1} - 1}_{m = m_j} \frac{\nu_m(s, a)}{\sqrt{N^+_{\tau (m)} (s, a)}} 
\end{equation}
It is important to observe that:
\begin{enumerate}
	\item It holds that $\nu_{m_j}(s, a) \leq N_{\tau (m_j)} (s, a)$, by the episode termination criteria of the \sw, 
	\item For $m \in \{m_j + 1, \ldots, m_{j+1} - 1\}$, we have $\sum^{m - 1}_{m' = m_j}\nu_{m'}(s, a) \leq N_{\tau (m)} (s, a)$ . Indeed, we know that episodes $m_j, \ldots, m_{j+1} - 1$ are inside the time interval $\{jW , \ldots, (j+1)W\}$. Consequently, the counts of ``$(s_t, a_t) = (s, a)$'' associated with $\{\nu_{m'}(s, a)\}^{m - 1}_{m' = m_j}$ are contained in the $W$ time steps preceding $\tau(m)$, hence the desired inequality. 
\end{enumerate}
With these two observations, we have 
\begin{align}
\eqref{eq:pf_sto_error_1st} &\leq \frac{\nu_{m_j}(s, a)}{\sqrt{\max\{\nu_{m_j}(s, a), 1\}}} + \sum^{m_{j+1} - 1}_{m = m_j + 1} \frac{\nu_m(s, a)}{\sqrt{ \max\{ \sum^{m - 1}_{m' = m_j}\nu_{m'}(s, a) ,  1\} }} \nonumber\\
&\leq \sqrt{\max\{\nu_{m_j}(s, a), 1\}} + (\sqrt{2} + 1)\sqrt{\max\left\{\sum^{m_{j+1} - 1}_{m = m_j} \nu_{m}(s, a), 1\right\}} \label{eq:pf_sto_error_2nd}\\
&\leq (\sqrt{2} + 2)\sqrt{\max\left\{\sum^{(j+1)W - 1}_{t = jW} \mathbf{1}((s_t,a_t) = (s, a)), 1\right\}}\label{eq:pf_sto_error_3rd}.
\end{align}
Step (\ref{eq:pf_sto_error_2nd}) is by Lemma 19 in \citep{JakschOA10}, which bounds the sum in the previous line. Step (\ref{eq:pf_sto_error_3rd}) is by the fact that episodes $m_j, \ldots, m_{j+1} - 1$ partitions the time interval $jW, \ldots, (j+1)W  -1$, and $\nu_{m}(s, a)$ counts the occurrences of $(s_t, a_t) = (s, a)$ in episode $m$. Finally, observe that $(\ref{eq:pf_sto_error_1st}) = 0$ if $\nu_m(s, a) = 0$ for all $m\in \{m_j, \ldots, m_{j + 1} - 1\}$. Thus, we can refine the bound in 
(\ref{eq:pf_sto_error_3rd}) to 
\begin{equation}\label{eq:pf_sto_error_4th}
\eqref{eq:pf_sto_error_1st}\leq (\sqrt{2} + 2)\sqrt{\sum^{(j+1)W - 1}_{t = jW} \mathbf{1}((s_t,a_t) = (s, a))}.
\end{equation}
The required inequality (\ref{eq:pf_sto_error_suffice}) is finally proved by summing (\ref{eq:pf_sto_error_4th}) over $s\in \SSS a\in \AAA$ and applying Cauchy Schwartz:
\begin{align*}
&\sum^{m_{j+1} - 1}_{m = m_j} \sum^{\tau(m + 1) - 1}_{t = \tau(m)} \frac{1}{\sqrt{N^+_{\tau (m)} (s_t, a_t)}} = \sum_{s\in \SSS, a\in \AAA_s}\sum^{m_{j+1} - 1}_{m = m_j} \frac{\nu_m(s, a)}{\sqrt{N^+_{\tau (m)} (s, a)}}\nonumber\\ 
\leq & (\sqrt{2} + 2) \sum_{s\in \SSS, a\in \AAA_s} \sqrt{\sum^{(j+1)W - 1}_{t = jW} \mathbf{1}((s_t,a_t) = (s, a))} \nonumber\\
\leq &(\sqrt{2} + 2)\sqrt{SAW} \nonumber.
\end{align*}
Altogether, the Lemma is proved.
\section{Proof of Theorem \ref{thm:borl}}\label{sec:thm:borl}
To begin, we consider the following regret decomposition, for any choice of $(W^{\dag},\eta^{\dag})\in J,$ we have
\begin{align}
\nonumber\text{Dyn-Reg}_T(\texttt{BORL})=&\sum_{i=1}^{\lceil T/H\rceil}\E\left[\sum_{t=(i-1)H+1}^{i\cdot H\wedge T}\rho_t^*-\reward_i\left(W_i,\eta_i,s_{(i-1)H+1}\right)\right]\\
\nonumber=&\sum_{i=1}^{\lceil T/H\rceil}\E\left[\sum_{t=(i-1)H+1}^{i\cdot H\wedge T}\rho_t^*-\reward_i\left(W^{\dag},\eta^{\dag},s_{(i-1)H+1}\right)\right]\\
\label{eq:borl_regret1}+&\sum_{i=1}^{\lceil T/H\rceil}\E\left[\sum_{i=1}^{\lceil T/H\rceil}\reward_i\left(W^{\dag},\eta^{\dag},s_{(i-1)H+1}\right)-\reward_i\left(W_i,\eta_i,s_{(i-1)H+1}\right)\right].
\end{align}
For the first term in eqn. \eqref{eq:borl_regret1}, we can apply the results from Theorem \ref{thm:main1} to each block $i\in\lceil T/H\rceil,$ \ie,
\begin{align}
\sum_{t=(i-1)H+1}^{i\cdot H\wedge T}\left[\rho_t^*-\reward\left(W^{\dag},\eta^{\dag},s_{(i-1)H+1}\right)\right]=&\tilde{O}\left(\frac{B_p(i)W^{\dag}}{\eta^{\dag}}+B_r(i) W^{\dag}+ D_\text{max}\left[B_p(i) W^{\dag} +\frac{S\sqrt{A}H}{\sqrt{W^{\dag}}} + H\eta^{\dag} + \frac{SAH}{W^{\dag}}  + \sqrt{H}\right]\right),\end{align}
where we have defined $$B_r(i)=\left(\sum_{t=(i-1)H+1}^{i\cdot H\wedge T}B_{r,t}\right),\quad B_p(i)=\left(\sum_{t=(i-1)H+1}^{i\cdot H\wedge T}B_{p,t}\right)$$ for brevity. For the second term, it captures the additional rewards of the DM were it uses the fixed parameters $(W^{\dag},\eta^{\dag})$ throughout w.r.t. the trajectory on the starting states of each block by the \borl,  \ie, $s_{1},\ldots,s_{(i-1)H+1},\ldots,s_{(\lceil T/H\rceil-1)H+1},$ and this is exactly the regret of the EXP3.P algorithm when it is applied to a $\Delta$-arm adaptive adversarial bandit problem with reward from $[0,H].$ Therefore, for any choice of $(W^{\dag},\eta^{\dag}),$ we can upper bound this by $$\tilde{O}\left(H\sqrt{\Delta T/H}\right)=\tilde{O}\left(\sqrt{TH}\right)$$ as $\Delta=O\left(\ln^2 T\right).$ Summing these two, the regret of the \borl~is 
\begin{align}
\label{eq:borl_regret2}
\text{Dyn-Reg}_T(\texttt{BORL}) = \tilde{O}\left(\frac{B_pW^{\dag}}{\eta^{\dag}}+B_r W^{\dag} + D_\text{max}\left[ B_p W^{\dag} + \frac{S\sqrt{A}T}{\sqrt{W^{\dag}}} + T\eta^{\dag} + \frac{SAT}{W^{\dag}}  + \sqrt{TH}\right]\right).
\end{align}
By virtue of the EXP3.P, the \borl~is able to adapt to any choice of $(W^{\dag},\eta{\dag})\in J.$ Note that 
\begin{align}
&H\geq W*=\frac{3S^{\frac{2}{3}}A^{\frac{1}{2}} T^{\frac{1}{2}}}{(B_r + B_p+1)^{\frac{1}{2}}}\geq \frac{3T^{\frac{1}{2}}}{(3T)^{\frac{1}{2}}}\geq1,\\
&S^{\frac{1}{3}}A^{\frac{1}{4}}\geq\eta^*=\frac{(B_p+1)^{\frac{1}{2}}S^{\frac{1}{3}}A^{\frac{1}{4}}}{(B_r + B_p+1)^{\frac{1}{4}}T^{\frac{1}{4}}}\geq\frac{S^{\frac{1}{3}}A^{\frac{1}{4}}}{2T^{\frac{1}{2}}}=S^{\frac{1}{3}}A^{\frac{1}{4}}\Phi.
\end{align}
Therefore, there must exists some $j^{\dag}$ and $k^{\dag}$ such that
\begin{align}
H^{j^{\dag}/\Delta_W}\leq W^*\leq H^{(j^{\dag}+1)/\Delta_W},\quad S^{\frac{1}{3}}A^{\frac{1}{4}}\Phi^{k^{\dag}/\Delta_{\eta}}\geq\eta^*\geq S^{\frac{1}{3}}A^{\frac{1}{4}}\Phi^{(k^{\dag}+1)/\Delta_{\eta}}
\end{align}
By adapting $W^{\dag}$ to $H^{j^{\dag}/\Delta_W}$ and $\eta^{\dag}$ to $\Phi^{k^{\dag}/\Delta_{\eta}},$ we further upper bound eqn. \eqref{eq:borl_regret2} as
\begin{align}
&\nonumber\text{Dyn-Reg}_T(\texttt{BORL})=\tilde{O}\left(D_{\max}(B_r+B_p+1)^{\frac{1}{4}}S^{\frac{2}{3}}A^{\frac{1}{2}}T^{\frac{3}{4}}\right).
\end{align} 
where we use $H^{1/\Delta_W}=\exp(1)$ and $\Phi^{1/\Delta_{\eta}}=\exp(-1)$ in the last step.

\section{Proof of Proposition \ref{prop:diameter}}\label{sec:prop:diameter}
We first show the following lemma.
\begin{lemma}\label{lemma:diameter_Hp}
	Conditioned on ${\cal E}_p,$ there exists a state transition distribution $p\in H_{p,t}(0)$ such that for every pair of states $s,s'\in\SSS,$ 
	$$p(s'|s,a_{(s,s')})\geq \zeta$$ for every time step $t\in[T].$
\end{lemma}
The proof of the lemma is provided in Section \ref{sec:lemma:diameter_Hp}. We then consider the state transition distribution $p\in H_{p,t}(0)$ specified in Lemma \ref{lemma:diameter_Hp}. For an arbitrary state $s'\in\SSS,$ we consider the policy $\pi$ such that $\pi(s)=a_{(s,s')}$ for all $s\in\SSS$ (see Assumption \ref{ass:alternative} for the definition of $a_{(s, s')}$). Starting from an arbitrary state $s\in\SSS,$ the policy either hits state $s'$ in the next step, which happens with probability at least $\zeta,$ or it transits to another state $s''\neq s',$ which would then hit state $s'$ in the next step with probability at least $\zeta.$ Therefore, the hitting process stochstically dominates the Bernoulli process with success probability $\zeta,$ and thus the expected hitting time is at most $1/\zeta.$
\subsection{Proof of Lemma \ref{lemma:diameter_Hp}}\label{sec:lemma:diameter_Hp}
First, we recall the definition of defintion of confidence region $H_{p,t}(s,a;0)$ in eqn. \ref{eq:Hp},
\begin{equation*} 
H_{p, t}(s, a; 0) = \left\{ \dot{p} \in \Delta^\SSS :\left\| \dot{p}(\cdot|s,a) - \hat{p}_t (\cdot | s, a) \right\|_1 \leq \rad_{p,t}(s, a)\right\}.
\end{equation*}
For every pair of states $s,s'\in\SSS,$ we construct $p$ by distinguishing the following two cases:
\begin{itemize}
	\item If $N_t(s,a_{(s,s')})=0,$ then by definition, $\rad_{p,t}(s, a_{(s,s')})\geq1,$ therefore every probability distribution $\bar{p}$ on $\SSS$ belongs to $H_{p, t}(s, a ; 0)$. Setting $p(\cdot | s, a_{(s, s')}) = p_t(\cdot |s, a_{(s, s')})$ for any $t$ satisfies the requirement in the Lemma.
	\item If $N_t(s,a_{(s,s')})>0,$ then we know from Assumption \ref{ass:alternative} and eqn. \ref{eq:p_bar} that
		$$\bar{p}_t(s'|s, a_{(s,s')}) = \frac{1}{N^+_t(s, a_{(s,s')})} \sum^{t-1}_{q =  (\tau(m) - W) \vee 1 } p_q(s' | s, a_{(s,s')}) \mathbf{1}(s_q = s, a_q = a_{(s,s')})\geq\zeta.$$
	By conditioning on $\cal{E}_p$, we know that $\bar{p}_t(\cdot|s,a_{(s,s')})\in H_{p,t}(s,a_{(s,s')};0),$ and we can thus set $p(\cdot|s,a_{(s,s')})=\bar{p}_t(\cdot | s,a_{(s,s')}).$
\end{itemize}
Combining the above cases, the transition probability distribution $p$ satisfies the stated inequality in the Lemma, and we conclude the proof.

\section{Proof of Proposition \ref{prop:pse}}\label{sec:prop:pse}
We first show that for any time step $t\in[T],$ we have $\rho^*_t - \rho^{*\pse}_t=-l\cdot\E[X_t].$ From Section \ref{sec:mdp_lp}, we have $\rho^*_t$ is equal to the optimal value of the following linear program $\textsf{P}(r_t, p_t);$ while $\rho^{*\pse}_t$ is equal to the optimal value of the following linear program $\textsf{P}(r_t^{\pse}, p_t).$ The two linear programs has the same set of constraints, and follows from eqn. \eqref{eq:pse1}, the only difference is that the objective value of $\textsf{P}(r_t^{\pse}, p_t)$ is $l\cdot\E[X_t]$ more than that of $\textsf{P}(r_t, p_t)$ (note that the summation of $x(s,a)$ over $s\in\SSS$ and $a\in\AAA_s$ is 1 from the second constraint of the linear program \eqref{eq:primal}). Therefore, we have
\begin{align}\label{eq:pse2}\sum_{t=1}^T\left(\rho^*_t - \rho^{*\pse}_t\right)=\sum_{t=1}^T-l\cdot\E[X_t].\end{align}

Next, conditioned on any demand realizations $X_1,\ldots,X_T,$ we can show by induction that the trajectory generated by $\Pi$ on $\mathcal{M}$ and $\mathcal{M}^{\pse}$ are exactly the same as they use the same sequence of state transition distributions. Therefore, 
\begin{align}\label{eq:pse3}
	\sum_{t=1}^T\E[r_t(s_t(\mathcal{M}),a_t(\mathcal{M}))|\{X_s\}_{s=1}^t]-	\sum_{t=1}^T\E[r^{\pse}_t(s_t(\mathcal{M}^{\pse}),a_t(\mathcal{M}^{\pse}))|\{X_s\}_{s=1}^t]=\sum_{t=1}^T-l\cdot X_t.
\end{align}
Taking expectation on both sides of eqn. \eqref{eq:pse3}, and combining this with eqn. \eqref{eq:pse2}, we can conclude the statement.
\end{APPENDIX}
\end{document}